\documentclass[10pt, conference, letterpaper]{IEEEtran}
\usepackage[font=footnotesize]{caption}
\IEEEoverridecommandlockouts
\usepackage{cite}
\usepackage{amsthm}
\usepackage{amsmath,amssymb,amsfonts}
\usepackage[cmintegrals]{newtxmath}
\usepackage{algorithm}
\usepackage{algorithmicx}
\usepackage{enumerate}
\usepackage{algpseudocode}
\usepackage{caption}
\usepackage{bm}
\usepackage{multicol} 
\usepackage{multirow}
\usepackage{booktabs}
\usepackage{graphicx}
\usepackage{array}
\usepackage{textcomp}
\usepackage{xcolor}
\usepackage{url}
\usepackage{tabularx}
\usepackage{color}
\usepackage{hyperref}
\usepackage{diagbox}
\usepackage{subfigure} 

\newlength{\whilewidth}
\algdef{SE}[parWHILE]{parWhile}{EndparWhile}[1]
  {\parbox[t]{\dimexpr\linewidth-\algmargin}{%
     \hangindent\whilewidth\strut\algorithmicwhile\ #1\ \algorithmicdo\strut}}{\algorithmicend\ \algorithmicwhile}%
\algnewcommand{\parState}[1]{\State%
  \parbox[t]{\dimexpr\linewidth-\algmargin}{\strut #1\strut}}

\newtheorem{theorem}{Theorem}

\newtheorem{Assumption}{Assumption}

\begin{document}
\title{Invariant Federated Learning for Edge Intelligence: Mitigating Heterogeneity and Asynchrony via Exit Strategy and Invariant Penalty}  
 \author{Ziruo Hao,
Zhenhua~Cui, 
Tao~Yang,~\IEEEmembership{Member,~IEEE,}
\\
Xiaofeng~Wu,
Hui~Feng,~\IEEEmembership{Member,~IEEE,}
and~Bo~Hu,~\IEEEmembership{Member,~IEEE}      

\thanks{Z. Hao, Z. Cui, T. Yang (Corresponding author), X. Wu, H. Feng and B. Hu are with the Department of
Electronic Engineering, School of Information Science and Technology,
Fudan University, Shanghai 200438, China (e-mail: zrhao22@m.fudan.edu.cn, czh773742344@sina.com, \{taoyang, xiaofeng\_wu, hfeng, bohu\}@fudan.edu.cn). H. Feng and B. Hu are also with the the Shanghai Institute of Intelligent Electronics and Systems, Shanghai 200433, China.
}
}

\maketitle

\begin{abstract}
This paper provides an invariant federated learning system for resource-constrained edge intelligence. This framework can mitigate the impact of heterogeneity and asynchrony via exit strategy and invariant penalty. We introduce parameter orthogonality into edge intelligence to measure the contribution or impact of heterogeneous and asynchronous clients. It is proved in this paper that the exit of abnormal edge clients can guarantee the effect of the model on most clients. Meanwhile, to ensure the models' performance on exited abnormal clients and those who lack training resources, we propose Federated Learning with Invariant Penalty for Generalization (FedIPG) by constructing the approximate orthogonality of the invariant parameters and the heterogeneous parameters. Theoretical proof shows that FedIPG reduces the Out-Of-Distribution prediction loss without increasing the communication burden. The performance of FedIPG combined with an exit strategy is tested empirically in multiple scales using four datasets. It shows our system can enhance In-Distribution performance and outperform the state-of-the-art algorithm in Out-Of-Distribution generalization while maintaining model convergence. Additionally, the results of the visual experiment prove that FedIPG contains preliminary causality in terms of ignoring confounding features.

\end{abstract}

\begin{IEEEkeywords}
Invariant learning, Asynchronous Federated Learning, Parameter Orthogonality, Heterogeneous environment, Out-of-Distribution Generalization, Edge intelligence
\end{IEEEkeywords}

\begin{figure*}[!t]
\centering
\includegraphics[width=7in]{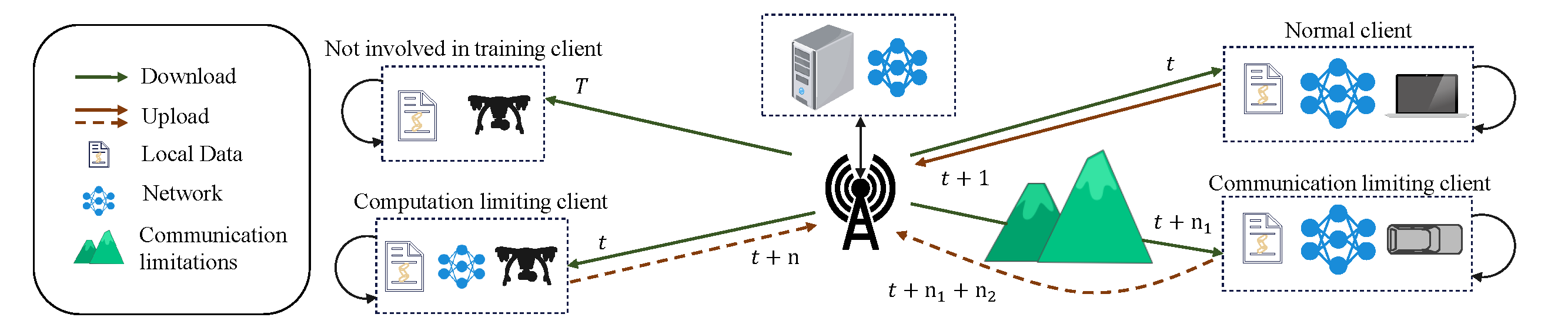}
\captionsetup{font={small}}
\caption{Edge intelligence scenarios with diverse computing and communication capabilities}
\label{edge}
\end{figure*} 
\section{Introduction}
\IEEEPARstart{A}{s} smart devices become increasingly lightweight and the demand for terminal task computing grows, the intelligence of edge devices has emerged as a prominent trend. With large volumes of data being collected and stored locally on these devices, there is an escalating need for a distributed learning framework to address communication constraints while ensuring robust privacy protection. In response to these challenges, federated learning (FL) has emerged as a key area of research.FL is a distributed computing paradigm designed to preserve privacy while enabling collaborative learning. The central tenet of FL is that data remains on local devices, with minimal interaction with a central server or other clients\cite{LiLi2020,Chen2021}. 

Although previous research has shown that, under ideal conditions, FL can achieve performance comparable to centralized learning, in edge intelligence contexts these ideal conditions are challenging to achieve, as shown in Fig.\ref{edge}. First, data collected from diverse sources typically exhibit significant heterogeneity and environmental interference. Second, in sophisticated scenarios involving edge intelligence, there exists a situation where communication and computational capabilities are unevenly distributed, ensuring synchronization within edge intelligence systems poses a considerable challenge. Both data heterogeneity and asynchrony can substantially degrade the performance of FL in such environments\cite{Xiaoxiao2021,Qu2022}. While existing works have focused on mitigating the negative impacts of heterogeneous and asynchronous clients through improved aggregation schemes, our approach replenish a new perspective on modeling and analysis.

 We propose an innovative framework where abnormal clients detected during the pretraining phase can autonomously withdraw from the training process, thereby alleviating the impacts of system heterogeneity and communication asynchrony. The trained model will be only employed on those aberrant clients after training. Our previous work has demonstrated that when latency is high, it is often more advantageous to allow computationally inefficient clients with poor communication environments to withdraw from the federated training process\cite{Impact}. In this article, we decompose the local optimal parameters of clients into two parameter vectors: the globally practical invariant component and the locally effective environmental heterogeneous component\cite{DFL}. The locally environmental nuisance parameters are required to be orthogonal to the parameter of invariant interest\cite{orthogonality}. We theoretically prove that as the size of the heterogeneous component increases and delays become more pronounced, anomalous clients are likely to negatively impact global aggregation in FedAvg\cite{FedAvg}. The proposed exit strategy thus protects the model from disturbances caused by such clients. Provided that the model can generalize well, their withdrawal can improve performance for the majority of regular clients without compromising the utility of the model for the few clients that exit.

 Although the exit strategy enhances performance for In-Distribution (ID) clients, it raises concerns regarding Out-Of-Distribution (OOD) generalization for exited clients and clients who lacking sufficient data or computational resources but still requiring robust OOD inference. This challenge has been addressed in centralized learning ~\cite{Ramé2023,Lee2023,OOD}, with invariant learning emerging as a prominent solution~\cite{InvariantLearning}. Invariant learning decouples data features into invariant (stable and effective across conditions) and environmental components, where environmental features can introduce noise or instability~\cite{Jonas2016}. For instance, in wildfire detection, flames and temperature represent invariant features, whereas the background (e.g., smoke or fog) can confound the classification~\cite{Mahajan2021}. Although several centralized schemes have leveraged this idea, applying it within FL remains limited~\cite{Shi2022,Ghosh2019}.

Inspired by Invariant Risk Minimization\cite{IRM} and parameter orthogonality mentioned earlier, we propose FedIPG to achieve OOD generalization in edge intelligence scenarios. FedIPG employs gradients to approximate environmental heterogeneous components and constructs an invariance penalty term to suppress the influence of environment-specific variations, thereby facilitating the learning of globally invariant parameters.This process, firmly grounded in parameter orthogonality, ensures that the environmental component does not undermine the estimation of invariant features, resulting in a robust model even for clients that exit the training process.By combining the exit strategy with FedIPG, our framework efficiently allocates communication resources to clients with stronger computational capabilities, leading to a globally invariant model that performs well on both ID and OOD data while safeguarding against fairness concerns. The key contributions and innovations of this paper are as follows: 
\begin{itemize}
\item We have theoretically demonstrated that when the total amount of training data is relatively sufficient, a heterogeneous or asynchronous client can be detrimental rather than beneficial to global aggregation. The exit strategy can optimize the model's performance on data from normal clients.

\item We introduce an invariant learning paradigm into the FL framework and propose a novel algorithm, \textbf{Fed}erated Learning with \textbf{I}nvariant \textbf{P}enalty for \textbf{G}eneralization (FedIPG), which enhances OOD generalization. Our approach is built upon the principle of parameter orthogonality, ensuring effective disentanglement of invariant and heterogeneous components.

\item We empirically demonstrate that FedIPG combined with the exit strategy improves performance on both ID and OOD data, outperforming baseline algorithms. Moreover, FedIPG shows promising resistance to confounding in causal learning scenarios.

\end{itemize}
\section{Related Work}
In the field of federated learning, numerous studies have addressed the challenges of non-IID data distributions~\cite{Shi2022} and asynchronous updates~\cite{you2022triple, wang2022gradient}. To mitigate these issues, several works focus on selecting a specific subset of clients to reduce the negative effects of outdated gradients~\cite{cui2023data, zhu2022online}. However, such methods typically require real-time monitoring of clients' channel conditions, a task that is particularly challenging in edge-intelligent environments, where the channel state is highly volatile~\cite{zhang2021client}. Moreover, no prior work has provided a theoretical quantification of the potential harm caused by the participation of heterogeneous and asynchronous clients using a parameter orthogonal decomposition approach, a principle that is critical for ensuring the asymptotic independence of the estimates for the invariance parameter and the environmental heterogeneity parameter. This orthogonality minimizes interference between these two components, leading to more stable and robust statistical inferences~\cite{orthogonality}. 

In centralized learning, methods for out-of-distribution (OOD) generalization often rely on environment labels to guide feature or representation learning. For example, Invariant Causal Prediction (ICP)~\cite{Pfister2019} leverages invariance properties to infer causal structures~\cite{Jonas2016}, though its effectiveness is highly dependent on the availability and quality of environment information~\cite{Liu2021}. Other approaches, such as the Causal Semantic Generative model (CSG)~\cite{LiuChang2021} and Causal Invariant Transformation (CIT)~\cite{Wang2022}, respectively extract semantic factors to bound generalization errors and modify non-causal features while preserving causal ones. Methods like Risk Extrapolation (REx) and its variant, MM-REx, trade off robustness to distribution shifts with stability against covariate variations~\cite{Krueger2021}, while Invariant Risk Minimization (IRM)~\cite{IRM} proposes to estimate nonlinear invariant predictors from multiple environments. Despite the promise of IRM, its non-convex bi-level formulation poses significant optimization challenges, and its performance may degrade under certain scenarios compared to empirical risk minimization (ERM)~\cite{Krikamol2013, Elan2020}. Importantly, these centralized methods do not account for the inherent data isolation present in FL systems. 

In the federated domain, recent works such as FedADG~\cite{Zhang2021} and FedSR~\cite{Nguyen2022} have focused on domain generalization through distribution alignment—usually at the cost of increased modeling complexity and additional computational overhead. Meanwhile, FedGen~\cite{FedGen} employs distributed masking functions with invariant learning to achieve improvements in efficiency, accuracy, and generalization for sequential data, and FedIIR~\cite{FedIIR} implicitly learns invariant relationships through early-round gradient exchange. However, the latter incurs additional communication overhead, rendering it less suitable for edge-intelligent scenarios with constrained resources. In contrast to these works, our study uniquely quantifies the detrimental effects of including heterogeneous and asynchronous clients through a parameter orthogonal decomposition framework and investigates improving the generalization of the model on clients not involved in the training by using parameter orthogonality without any additional communication in formal training process.

\begin{figure*}[!t]
\centering
\includegraphics[width=7in]{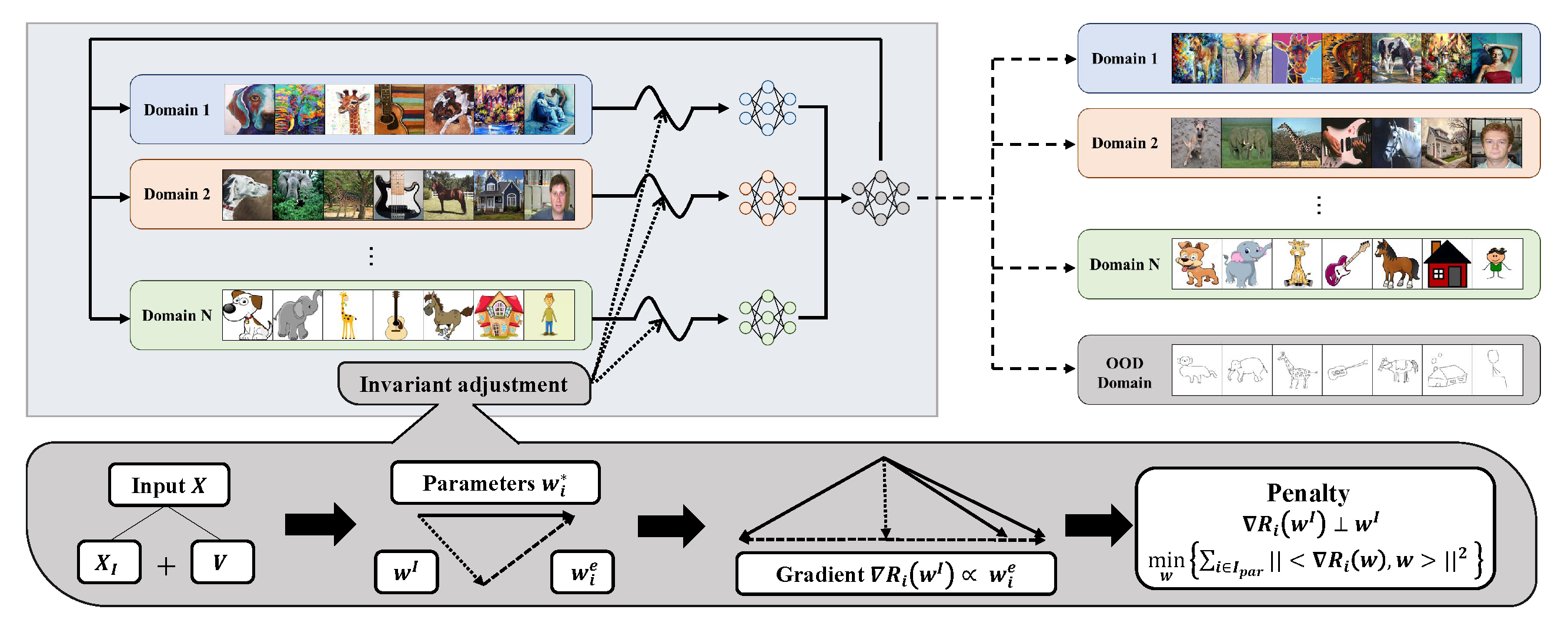}
\captionsetup{font={small}}
\caption{The process of federated learning in edge intelligence scenarios with guaranteed OOD generalization}
\label{training}
\end{figure*} 
\section{The System Model of Invariant Federated Learning }
\subsection{The Federated System Model with Exit Strategy}
\begin{algorithm}[t]
  \captionsetup{font={small}}
  \caption{\emph{Federated Learning with Invariant Penalty for Generalization with exit strategy}}
  \label{AFL algorithm}
  \begin{algorithmic}[1]
  \linespread{1}\selectfont
  \State \textbf{Pretraining:} Identify $\mathcal{I}_{ID} \in \mathcal{I}_{all}$
  \State \textbf{Initialization:} 
  \State Set $t=1$
  \State The center sends the initialized model parameters $w^1$ to all nodes
  \State \textbf{Center Side:}
  \For {$t < T+1$}
  \State \textbf{Aggregation}: $w^{t+1} = \sum_{i\in \mathcal{I}_{ID}} \frac{n_i}{\sum_{j\in \mathcal{I}_{ID}}n_j}w^{t}_i$.
  \State \textbf{Distribute} $w^{t+1}$ to the client
  \State Set $t = t+1$
  \EndFor
  \State \textbf{Client Side:}
  \For {$t < T+1$}
  \For {all clients: $i={1,2,..,N}$}
  \State \textbf{Receive} $w^t$ 
  \State \textbf{Compute} $w^{t}_i = w^{t} - \eta\nabla f_i(w^{t})$
  \State \textbf{Upload} $w^{t}_i$ to the server
  \EndFor
  \EndFor
  \For {$t = T+1$}
  \For {all clients: $i\in \mathcal{I}_{all} $}
  \State \textbf{Apply} $w^{t}_i$ to prediction tasks
  \EndFor
  \EndFor
  \end{algorithmic}

\end{algorithm}

In this paper, we consider an edge intelligent scenario where a total of $N_{all}=|\mathcal{I}_{all}|$ clients participate in training a unified task within an edge computing system, where $\mathcal{I}_{all}$ is collection of all clients, which is possibly infinite. Each client is given a dataset $\mathcal{D}_i =\{(x_k,y_k)\}^{n_i}_{k=1}, i=1,2...,N_{all}$, where $x_k\in \mathbb{R}^d$ is the input, $y\in\mathbb{R}^1$ is the label, and $n_i=|\mathcal{D}_i|$ denotes the number of samples in dataset $D_i$ located at client $i$.
In addition, the dataset in each client is assumed to have local heterogeneous features and global invariant properties based on the principle described in the last subsection, which means the distribution of any $\mathcal{D}_i$ is Non-IID and different from the distribution of $\sum_{i\in \mathcal{I}_{all}}\mathcal{D}_i$\cite{Alexandre2022}.
In this paper, according to the standard definitions in Federated Learning scheme, the empirical risk $R_i(w)$ based on training data of clients is defined as
\begin{equation}
    R_i(w)=\frac{1}{n_i}\sum_{x,y\in \mathcal{D}_i}\mathcal{L}(x,y,w),
\end{equation}
where $\mathcal{L}(x,y,w)$ is the empirical risk of a sample pair$(x,y)$ from local data $\mathcal{D}_i$ made with model parameters $w$\cite{Yiqiang2020}. The formula of $\mathcal{L}$ is kept consistent across localities, e.g. cross-entropy loss for classification tasks.

During each iteration, clients receive the global parameter $w^{t-1}$ transmitted from the central server. Based on $w_i^{(t,0)}=w^{t-1}$, the local entities optimize the parameters using gradient descent on their local data. Each client in $\mathcal{I}_{all}$ compute
\begin{align}
    w_i^{t+1}=w^t-\eta\nabla f_i(w^t)
    \label{update}
\end{align} 
where $\eta$ is the learning rate and $f_i(\cdot)$ represents a designed local loss function. Once the optimization process concludes, each local entity uploads its local parameter $w^{t}_i$ to the central server. 

Under ideal conditions, the central server aggregate the local gradients and applies the update $w^{t+1} = \sum_{i \in {\mathcal{I}_{all}}}\frac{n_i}{n}w_i^{t}$, where $n=\sum_{i \in \mathcal{I}_{all}}n_i$.
The global parameters $w^{t+1}$ is obtained and it will be distributed to each local entity for further processing and utilization. However, due to the objective constraints of clients' communication and computational capabilities, clients cannot fully synchronize their updates, leading to asynchronous errors in the global aggregated parameters caused by delays $\tau_i$ for client $i$. Thus, the aggregated parameter turns into $w^{t+1} = \sum_{i \in {\mathcal{I}_{all}}}\frac{n_i}{n}w_i^{t-\tau_i}$.
Despite the abundance of research focusing on asynchronous errors, these errors have an irremovable impact on model aggregation during the federated learning process. This impact may hinder learning efficiency and the final performance of the model. Therefore, in this paper, we propose an exit mechanism that selectively includes only clients in $\mathcal{I}_{ID}$ with sufficient computational, data, and communication resources in the learning process. Based on this premise, the paradigm for parameter aggregation takes the following form.
\begin{align}
    w^{t+1} = \sum_{i \in {\mathcal{I}_{ID}}}\frac{n_i}{n}w_i^{t}. 
\end{align}
For clients that are included in $\mathcal{I}_{all}$ but excluded from $\mathcal{I}_{ID}$, we refer to them as OOD clients in $\mathcal{I}_{OOD}$. While clients in $\mathcal{I}_{OOD}$ do not participate in training, our model ensures post-training performance on their datasets. The loss function balances accuracy on the training set with the invariance of parameters across environments. Specifically, we introduce an invariance penalty term to align parameters, ensuring they are effective on both par and ood datasets. Specifically, we design an $f_i(w)$ such that the following condition holds
\begin{equation}
    w^I = \min_w \sum_{i\in \mathcal{I}_{ID}}a_i f_i(w) \to \min_w \sum_{i\in \mathcal{I}_{all}}b_i R_i(w),
\end{equation}
Here, $a_i$ and $b_i$ represent the weights of each client in the global loss function, which are designed based on various evaluation metrics, including data volume and data value. To ensure the straightforward applicability of our algorithm, we employ data volume as a surrogate measure of client importance. It is important to note that approaches relying on data value estimation are not mutually exclusive with our design. Thus, the FL algorithm we consider is applicable to any finite-sum objective of the form
\begin{equation}
    f(w)=\sum_{i\in \mathcal{I}_{ID}}\frac{n_i}{n}f_i(w),
\end{equation}
where $f_i(w)$ is the loss function with penalty based on $R_i(w)$. In the next subsection, we will elaborate on the rationale behind the design of the penalty term and present the complete invariant federated learning algorithm.

\subsection{The Profile of the Parameter Orthogonality the Invariant Schemes}

In the traditional federated learning framework, each client focuses solely on minimizing the empirical risk, guiding the parameters towards local optimal values. At the server side, the aggregation of information from various clients yields a federated global parameter that approximates the optimal parameter obtained in centralized training. In the context of edge intelligence, we select clients with stable computational capabilities to participate in federated training. If the penalty term is computed based solely on local data, parameters, or gradients,  the penalty term would not increase communication costs which is good for edge computing.

Our objective is to steer the parameters towards emphasizing invariant features that maintain a causal relationship with the task across diverse client environments, thereby facilitating generalization to OOD data, as shown in figure \ref{training}. To achieve OOD generalization, the model needs to decouple the data $X$ into two independent features: the invariant features $X_I$ and the heterogeneous features $V$. Inspired by previous researchers \cite{orthogonalvectors}, we decompose the optimal parameter vector of each client into two orthogonal vectors, $w^I$ and $w^e_i$. Here, $w^I$ corresponds to $X_I$, which remains consistent across all clients, while $w^e_i$ corresponds to $V$, which varies from client to client. The orthogonality between $w^I$ and $w^e_i$ reflects the independence between $X_I$ and $V$. According to parameters orthogonality theorem\cite{orthogonality}, we define $w^I$ and $w^e_i$ as
\begin{align}
    E(\frac{\partial R_{i}(w)}{\partial w^e_i}\frac{\partial R_{i}(w)}{\partial w^I};w)=0, 
\end{align}
where $R_i(w^I+w^e_i)=0$. Since $-\nabla R_i(w)$ can be interpreted as pointing toward the optimal parameter when the current parameter is $w$, i.e., $-\nabla R_i(w) \approx w^*_i - w$. So, we can rewrite the relationship between $w^I$ and $w^e_i$ as: 
\begin{align}
   \forall i :\ E(< -\nabla R_i(w^I),w^I > ; w=w^I+w^e_i)=0, 
\end{align}

If $w = w^I$, then $-\nabla R_i(w^I) \perp w^I$. Consequently, our goal shifts to training a invariant parameter $w^I$ within the federated system. To identify the global invariant parameter, we need to ensure that the following equality holds

\begin{align}
    w^I\in \{w|\|<\nabla R_i(w),w>\|^2=0,\forall\ i \in\mathcal{I}_{ID} \}
\end{align}
Based on the above reasoning, the loss function of client $i$ in the FedIPG system is defined as:
\begin{equation}
f_i(w) = R_i(w)+\lambda \|<\nabla R_i(w),w>\|^2.
\end{equation}
Positive hyper-parameter $\lambda$ denotes the strength of invariant regularization. The quest of regularization $ \|<\nabla R_i(w),w> \|^2$ is to find invariant features. 

Significantly, each client's update still optimizes towards its local optimal parameter. Due to the invariance penalty term altering the distribution of the client loss functions, the clients exhibit a smoother landscape at the invariance optimal parameter point, where $<\nabla R(w), w> = 0$. After weighted summation at the server, the global optimal parameter $w^*$ can be considered the midpoint of the client optimal parameters $w^*_i$. Guided by the penalty term, the summation yields a new invariance optimal parameter $w^I$. $w^I$ exhibits relatively smooth gradients across any client, indicating a focus on invariance features while potentially neglecting rapidly changing environmental information.

Algorithm 1 shows the pseudo-code of FedIPG. In conclusion, each client locally takes a step of gradient descent based on the current global parameter and the invariant loss function calculated with local data. The server then takes a weighted average of the local parameter as the global parameter.

\begin{table}[t]
\centering
\captionof{table}{Summary of Notations}
\label{notations}
\renewcommand\arraystretch{1.5}
\begin{tabular}{c|l}
 \hline
    \textbf{Notation}                                             & \textbf{Definition}\\ \hline
    $w^t$  
            & The global parameters at $t^{th}$ training iteration\\ \hline
    $w^t_i$                      
            & The local parameters of $clinet_i$ at $t^{th}$ training iteration   \\ \hline
    $w^*$
            & The global optimal training parameters \\ \hline
    $w^I$
            & The global optimal invariant parameters \\ \hline
    $w^*_i$
            & The local optimal parameters of $clinet_i$ \\ \hline
    $\eta$                             
            & Learning rate \\ \hline
    $\lambda$
            & Hyper-parameters for the invariance penalty term \\ \hline        
    $T$
            & The total number of communication rounds\\ \hline
    $\phi$                         
            & The parameter for bounded data heterogeneity assumption\\ \hline
    $L$                         
            & The parameter for smoothness assumption\\ \hline
    $\mu$
            & The parameter for convexity assumption \\ \hline
    $B$
            & The parameter for compactness assumption\\ \hline
    $G$
            & The parameter for bounded gradient assumption\\ \hline
    $\rho$
            & The parameter for second-order gradient assumption\\ \hline
    
\end{tabular}
\end{table}
\section{The basis of Aberrant Clients Exit Strategy }
\subsection{Theoretical Measure for Clients' Contribution}
In this section, we define the contribution of a client's participation in a round of aggregation in FedAvg. When there is a probability that a client's contribution is negative, it implies that the client's participation may have a detrimental effect on the model's performance. By orthogonally decomposing parameters, the contributions of heterogeneous and asynchronous clients are quantified.

Firstly, we quantify the contribution of each client to the global aggregation. To facilitate the derivation of reliable theoretical results, we introduce several assumptions regarding the local loss functions.
\begin{Assumption}(Smoothness): \label{smoothness}
\\
The local loss function $R_i,\,i \in \{1,...,N\}$ is $L$-smooth if $\forall w_1,w_2$, the following inequality holds
\begin{equation}
R_i(w_1)-R_i(w_2) \le <\nabla{R_i(w_2)},{w_1-w_2}>+\frac{L}{2}{\|w_1-w_2\|}^2.
\end{equation}
\end{Assumption}
The prevalence and commonality of this smoothness assumption in the analysis of convergence in Federated Learning are well-established, as evident in studies\cite{chen2020asynchronous}. Additionally, we assume the local loss function is convex as discussed in \cite{fang2022convex}.
\begin{Assumption} (Convexity):\label{Convexity}
\\
The local loss function $R_i,\,i \in \{1,...,N\}$ is $\mu$-convex if $\forall w_1,w_2$, the following inequality holds 
\begin{equation}
R_i(w_1) - R_i(w_2) \geq <\nabla{R_i(w_2)},w_1-w_2> + \frac{\mu}{2} {\|w_1-w_2\|}^2.
\end{equation}
\end{Assumption}
Based on Assumption \ref{smoothness} and Assumption \ref{Convexity}, we have\cite{submodular}:
\begin{align}
R(w^T)-R(w^*)\leq \prod^{T-1}_{t=0}(1-2\mu \sum_{i \in \mathcal{I}_{ID}}C_i^t)(R(w^0)-R(w^*))
\end{align}
where 
\begin{align}
 C^t_i=\eta \frac{n_i}{n}\frac{ <\nabla R(w^t), \nabla R_i(w^{t-\tau_i^t})> - \frac{L}{2} \|\nabla R_i(w^{t-\tau_i^t})\|^2 }{\|\nabla R(w^t)\|^2}.
\end{align}
Here, \( C_i^t \) represents the contribution of client \( i \)'s participation in the aggregation at round \( t \) on the training process, where \( \tau_i^t \) denotes the asynchronous delay of client \( i \) at round \( t \). Client \( i \)'s participation in round \( t \) is beneficial to the aggregation if and only if \( C_i^t > 0 \). Next, we will conduct a detailed analysis of the variation of \(P_i^t\) under the influence of \(w_i^e\) and \(\tau\), aiming to elucidate the impact of these factors on system performance.

\subsection{Contribution Analysis for Aberrant clients}
First, we analyze the effect of strongly heterogeneous clients when $\tau_i^t = 0$. In common federated learning research, the distribution discrepancy assumption is often adopted, which characterizes the differences between data distributions\cite{Dissimilarity}.
\begin{Assumption}(Local Dissimilarity): \label{Bounded Local Dissimilarity}
    \begin{equation}
    U_i^2\|\nabla R(w)\|^2\leq\|\nabla R_i(w)\|^2\leq V_i^2\|\nabla R(w)\|^2
    \end{equation}
\end{Assumption}

Here, we assume the presence of a highly heterogeneous client \( i \), while the data of other clients can be approximated as independent and identically distributed (IID), i.e., \(U_i \approx V_i \approx 1\).  Under this assumption, the upper bound of \(C_i^t\) can be expressed as:
\begin{align}
	C^t_i=\eta \frac{n_i}{n} (U_i\cos{\theta_i^t} - \frac{L}{2} V_i^2)\label{valueofclient}
\end{align}
where \(\theta_i^t\) is the angle between two gradient vectors. namely $\nabla R(w^t)$ and $\nabla R_i(w^t)$, when the learning rate satisfy \(\eta \leq \frac{2n}{n_{min} L}\), in which $n_{min}$ is the data volume of client with smallest amount of data. Clearly, the magnitude of \(C_i^t\) is highly correlated with \(\theta_i^t\).

When this external heterogeneous client \( i \) participates in training, whose optimal parameter is \(w^*_i=w_I+w^e_i\), we obtain the following conclusion:
\begin{theorem} (The Contribution Analysis for Heterogeneous client)
\\
For a optimization started from the origin of coordinates, while \(w^t=\kappa \frac{w^I}{\|w^I\|} + \nu \frac{w^e_i}{\|w^e_i\|}\), $\kappa\in (0,1)$ and $\nu \in (-\epsilon,\epsilon)$, the contribution of client \( i \) in iteration $t$ can be represent as:
\begin{align}
	C^t_i=\eta \frac{n_i}{n} (\frac{U_i((1-\kappa)^2-\nu(\|w^e_i\|-\nu))}{\sqrt{(1-\kappa)^2+\nu^2}\sqrt{(1-\kappa)^2+(\|w^e_i\|-\nu)^2}} - \frac{L}{2} V_i^2)\label{heterclientvalue}
\end{align}
\begin{IEEEproof}
See Appendix-A.
\end{IEEEproof}
\end{theorem}
Under the condition of sufficient data volume, the impact of a single client on global aggregation can be considered minimal, which means \(\epsilon < \|w^e_i\|\). When the magnitude of \(\|w^I\|\) is fixed, the denominator of (\ref{heterclientvalue}) increases as \(\|w^e_i\|\) grows. For \(\nu \leq 0\), the numerator decreases with the increase of \(\|w^e_i\|\). For \(\nu \geq 0\), the numerator increases with \(\|w^e_i\|\), but at a slower rate than the denominator. Clearly, as data heterogeneity increases, the value of client participation in global aggregation diminishes. Clients with strong heterogeneity are more likely to have a negative impact.

Next, we analyze the impact of  delay $\tau$ to $C^t_i$. To simplify the problem, we assume all clients include highly homogeneous data, which means $\nabla R_i(w)$ is equal to  $\nabla R(w)$. The upper bound of \(C_i^t\) can be expressed as:
\begin{align}
	C^t_i&=\eta \frac{n_i}{n}\frac{ <\nabla R(w^t), \nabla R(w^{t-\tau_i^t})> - \frac{L}{2} \|\nabla R_i(w^{t-\tau_i^t})\|^2 }{\|\nabla R(w^t)\|^2} \notag \\
    &=\eta \frac{n_i}{n}\frac{ \|\nabla R(w^{t-\tau_i^t})\|}{\|\nabla R(w^t)\|}- \frac{L}{2} \frac{\|\nabla R_i(w^{t-\tau_i^t})\|^2 }{\|\nabla R(w^t)\|^2}
\end{align}
When \( \tau_i^t = 0 \), \( C_i^t > 0 \), which implies \( L < 2 \), a condition that typically holds under normal circumstances. Notably, \( C_i^t \) degenerates into a function of \( \frac{ \|\nabla R(w^{t-\tau_i^t})\|}{\|\nabla R(w^t)\|} \). We next analyze the trend of \( \frac{ \|\nabla R(w^{t-\tau_i^t})\|}{\|\nabla R(w^t)\|} \) as \(\tau\) changes, which indirectly reflects the variation in client \(C^t_i\). Most importantly, we examine the negativity of \(C^t_i\) when \(\tau\) is large. The squared norm of the gradients is uniformly bounded\cite{Yuchen2012}.
\begin{Assumption}
    \label{Bound}
    (Bounded gradient)\\
Each local loss function $ R_i$ , $i \in \mathcal{I}_{ID}$ is differentiable and there exists a constant $G\geq 0$ such that 
    \begin{align}
        \|\nabla R_i (w)\|^2 \leq G^2 .
    \end{align}
    It can be obtained from the mean inequality that
    \begin{align}
        \nabla R_i (w) \nabla R_i (w)^\top \preceq G^2 I.
    \end{align}
\end{Assumption}
Based on smoothness assumption and bounded gradient assumption, we obtain the following conclusion:
\begin{theorem} (The Contribution Analysis for Asynchronous client)
\\
For and certain parameters $w$, the upper bound of \( \frac{ \|\nabla R(w^{t-\tau})\|}{\|\nabla R(w^t)\|^2} \) related to a delay $\tau$  can be limited by following inequality
\begin{align}
	\frac{ \|\nabla R(w^{t-\tau})\|}{\|\nabla R(w^t)\|} \leq 1+\frac{L G \tau}{\|\nabla R(w^t)\|}
\end{align}
\begin{IEEEproof}
See Appendix-A.
\end{IEEEproof}
\end{theorem}
If \( L \in (1, 2) \), as \( \tau_i^t \) increases, the potential minimum of \( C_i^t \) will continuously decrease. If \( L \in (0, 1) \), as \( \tau_i^t \) increases, the potential minimum of \( C_i^t \) will first increase and then decrease, which aligns with conclusions from some prior studies. When \( \tau_i^t \to \infty \), regardless of the value of \( L \), the contribution of this asynchronous client will inevitably be negative.

As the magnitude of a client's heterogeneous component and its delay increase, the likelihood of its contribution being negative also rises. This observation motivates a scheme in which aberrant clients proactively opt out of the training process rather than waiting for the server to exclude them. Prior to formal training, we assume that all clients and the server jointly conduct a pre-training phase, modeled similarly to the design in \cite{submodular}, that assesses each client's delay and data quality. Based on this self-assessment, clients can decide whether to participate in subsequent training rounds. Notably, similar contribution metrics can be incorporated into asynchronous or heterogeneous federated settings, indicating that our approach complements and extends existing research rather than contradicting it.

\section{Generalization and Convergence Analysis of FedIPG} 
Based on the analysis in the previous section, some clients will drop out of the training process, while others inherently lack the capability to participate in training. These clients will directly utilize the trained model upon completion of the training. This section will provide a detailed discussion on the invariance correction implemented to enhance the model's performance on the aforementioned datasets, as well as the convergence properties of the model under such invariance correction.
\subsection{Out-of-Distributed Generalization by Invariant Penalty Term}
First, based on the smoothness and convex assumptions, we show that the penalty term helps to improve the performance of the model in the generalization data set in the affine space.

We denote an expected risk as $R_{OOD}=\mathbb{E}_{i\in \mathcal{I}_{OOD}}R_i$. We assume that there always exist $\xi \in \Xi_\upsilon:=\{\{\xi_i:i\in \mathcal{I}_{ID}\ |\xi_i\geq-\upsilon,\sum_{i\in\mathcal{I}_{ID}}\xi_i =1\}\}$ makes $\mathbb{P}_{OOD}=\sum_{i\in\mathcal{I}_{ID}}\xi_i\mathbb{P}_{i}$, which means the distribution of $\mathcal{I}_{ID}$ can be seen as a affine combination of the distribution of participants $\mathbb{P}_{ID}$ . This assumption about datasets is common in prior studies\cite{FedIIR}. Under the above assumptions, we prove in the following theory that the training system with the invariant penalty will improve the performance of the model on the generalized data set.

\begin{theorem} (The Upper Bound of Loss Functions in Generalized Data-set)
\\
If the collection of clients that participate training is $\mathcal{I}_{ID}$, then the upper bound of expected risk on the generalized distribution can be described as: 
\begin{align}
R_{OOD}(w) \leq& R(w)+(1+|\mathcal{I}_{ID}|\upsilon )\sup_{i,j\in\mathcal{I}_{ID}}\left(R_i(0)-R_{j}(0)\right) \notag\\
&+2(1+|\mathcal{I}_{ID}|\upsilon)\sup_{i\in\mathcal{I}_{ID}}<\nabla R_i(w),w>\label{OODupperbound}
\end{align}

where $i$ and $j$ are any clients in $\mathcal{I}_{ID}$.
\begin{IEEEproof}
Based on the assumption of affine combination and the procedural conclusions of related studies\cite{affine}, we have
\begin{align}
    R_{OOD}(w)=&\sup_{\xi \in \Xi}\sum_{i\in\mathcal{I}_{ID}}\xi_cR_i(w) \notag \\
    \leq&(1+|\mathcal{I}_{ID}|\upsilon)\sup_{i\in\mathcal{I}_{ID}}R_i(w)-\upsilon \sum_{i\in\mathcal{I}_{ID}}R_i(w) \notag\\
    =&\frac{1}{|\mathcal{I}_{ID}|}\sum_{i\in\mathcal{I}_{ID}}R_i(w)+(1+|\mathcal{I}_{ID}|\upsilon)\sup_{i\in\mathcal{I}_{ID}}R_i(w)\notag\\
    &-(1+|\mathcal{I}_{ID}|\upsilon)\sum_{i\in\mathcal{I}_{ID}}R_i(w)\notag \\
    =&R(w)+(1+|\mathcal{I}_{ID}|\upsilon)\left(\sup_{i\in\mathcal{I}_{ID}}R_i(w)-R(w)\right)\notag \\
    \leq&R(w)+(1+|\mathcal{I}_{ID}|\upsilon)\sup_{i,j\in\mathcal{I}_{ID}}\underbrace{\left(R_i(w)-R_j(w)\right)}_{A}
\end{align}
According to the smoothness and convex assumptions, there exist a upper bound of $R_i(w)-R_{j}(w)$ as:
\begin{align}
    A\leq&\left(R_i(0)+<\nabla R_i(w),w>-\frac{L}{2}\|w\|^2\right)\notag \\
    &-\left(R_j(0)+<\nabla R_j(w),w>-\frac{\mu}{2}\|w\|^2\right)\notag \\
    =&R_i(0)-R_j(0)-\frac{L-\mu}{2}\|w\|^2 \notag \\
    &+\left(\nabla R_i(w),w>-<\nabla R_j(w),w>\right) \notag \\
    \leq &R_i(0)-R_j(0)-\frac{L-\mu}{2}\|w\|^2+2\sup_{i\in\mathcal{I}_{ID}}<\nabla R_i(w),w>
\end{align}
Plugging the bound on A, we obtaining
\begin{align}
    R_{ood}(w) \leq& R(w)+(1+|\mathcal{I}_{ID}|\upsilon )\sup_{i,j\in\mathcal{I}_{ID}}\left(R_i(0)-R_{j}(0)\right) \notag\\
    &-\frac{L-\mu}{2}(1+|\mathcal{I}_{ID}|\upsilon)\|w\|^2 \notag\\
    &+2(1+|\mathcal{I}_{ID}|\upsilon)\sup_{i\in\mathcal{I}_{ID}}<\nabla R_i(w),w>\\
    \leq& R(w)+(1+|\mathcal{I}_{ID}|\upsilon )\sup_{i,j\in\mathcal{I}_{ID}}\left(R_i(0)-R_{j}(0)\right) \notag\\
    &+2(1+|\mathcal{I}_{ID}|\upsilon)\sup_{i\in\mathcal{I}_{ID}}\underbrace{<\nabla R_i(w),w>}_{invariant\ penalty}
\end{align}

\end{IEEEproof}
\end{theorem}
In Equation (\ref{OODupperbound}), $\mathcal{R}_i(0)-\mathcal{R}_{j}(0)$ is an inherent property of client data that does not change with change of parameters. So, by minimizing $<\nabla R(w) , w>$ in the last term, the parameter with better performance on the generalized dataset can be found. Additionally, as seen in Equation  (\ref{OODupperbound}), the maximum upper bound is related to $R_i(0)-R_{j}(0)$, the maximum difference in the client loss functions at the zero point. The exit strategy for heterogeneous clients reduces $R_i(0)-R_{j}(0)$, thereby significantly lowering this upper bound.

\subsection{Convergence Analysis of Invariant Asynchronous Federated Learning Algorithm}

In addition to Assumption \ref{smoothness} and  Assumption \ref{Convexity} mentioned above, we also introduce some assumptions for theoretical analysis. A standard compactness assumption is incorporated into the subsequent analysis, which is established by existing studies, and refers to the property of parameters restricted to a compact set\cite{Dmitrii2021}. 
\begin{Assumption}
    \label{Compactness}
    (Compactness)\\
    For $w^t_k$ there exists a constant $\beta\leq 0$ and $R \leq 0$ such that
    \begin{align}
        \|w^t_i-w^*\|^2 \leq B^2.
    \end{align}
    \label{Compactnessmatrix}
    Based on the (\ref{Compactness}), the property of the parameters can be deduced to has certain compactness. $\forall w,$ we have
    \begin{align}
        w^t_i {w^t_i}^\top \preceq \beta^2 I.
    \end{align}
\end{Assumption}
We also assume a customary quality for the Hessian of all local empirical risks $R_i(w)$, a common assumption analyzing the second-order derivative\cite{Fallah2020}.
\begin{Assumption}
    \label{Third-order}
    (Second-order gradient $\rho -Lipschitz$)\\
    Each local loss function $ R_i$ , $i \in \mathcal{I}_{ID}$ is $\rho -Lipschitz$, which means there exists a constant $\rho \geq 0$ such that 
    \begin{align}
        \|\nabla^2 R_i(w_1)-\nabla^2 R_i(w_2)\|\leq \rho \|w_1-w_2\|
    \end{align}
    \label{rho}
\end{Assumption}
We next formally prove that the model with invariant penalty results in an aggregated model by the above Assumptions.

First, we demonstrate that the local loss function, after incorporating the penalty term, retains properties similar to those in Theorem \ref{smoothness} and Theorem \ref{Convexity}.
\begin{theorem}
\label{convex'}
($\mu'$-Convexity)\\
With $0\leq\lambda\leq \frac{\mu-\mu'}{8\theta_k\mu^2+2G\rho\beta^2 }$, there exists a constant $\mu'$ , $\mu\geq\mu'\geq0$, such that
\begin{equation}
\begin{aligned}
        f_k(w_1)-f_k(w_2)\geq &< \nabla f_k(w_2),w_1-w_2 > \\&+ \frac{\mu'}{2}\|w_1-w_2\|^2
    \end{aligned}
    \label{convex-mu}
\end{equation}
where, $\theta_k = R_i(0)-R_i(w^*)$. The optimal invariant solution of function $f_k(w)$ is the optimal solution $w^*_k$ of function $R_k(w)$.
\end{theorem}

\begin{proof}
    \begin{equation}
    \begin{aligned}
      \nabla f_i(w) = \nabla R_i(w) + 2 \lambda (w^T \nabla R_i(w)) (\nabla R_i(w) + \nabla^2 R_i(w) w)
    \end{aligned}
    \end{equation}
    The second derivative of the loss function including the penalty term can be obtained as
    \begin{equation}
    \begin{aligned}
        &\nabla^2 f_i(w) = \nabla^2 R_i(w) \\
        & + 2 \lambda \left[ (w^T \nabla R_i(w)) (\nabla^2 R_i(w) + \nabla^3 R_i(w) w)\right]\notag \\
        & + 2\lambda \left[(\nabla R_i(w) + \nabla^2 R_i(w) w)(\nabla R_i(w) + \nabla^2 R_i(w) w)^T \right]
    \end{aligned}
    \end{equation}
    According to (\ref{Convexity}), When $w_1=0$, we have
    \begin{align}
        \nabla R_i(w)^\top w\geq R_i(w)-R_i(0)+\frac{\mu}{2}w^\top w.
    \end{align}
    Substituting into above equation
    \begin{equation}
    \begin{aligned}
        \nabla^2 f_i(w) =& \nabla^2 R_i(w) + 2\lambda (w^T \nabla R_i(w)) \nabla^2 R_i(w) \\
        &+ 2\lambda (w^T \nabla R_i(w)) \nabla^3 R_i(w) w \\
        &+ 2\lambda \nabla R_i(w) \nabla R_i(w)^T + 2\lambda \nabla R_i(w) (\nabla^2 R_i(w) w)^T \\
        &+ 2\lambda \nabla^2 R_i(w) w \nabla R_i(w)^T \\
        &+ 2\lambda (\nabla^2 R_i(w) w)(\nabla^2 R_i(w) w)^T
    \end{aligned}
    \end{equation}
     The smallest eigenvalue of the addition of two matrices is greater than the sum of the smallest eigenvalues of the two matrices. The smallest eigenvalue of the second, third and fourth terms on the right-hand side of the inequality is clearly zero. The absolute values of the eigenvalues of the fifth and sixth terms are smaller than their norms. 
     
    Based on (\ref{rho}), we can easily have
    \begin{equation}
    \begin{aligned}
    \|\nabla^3 R_i(w)\| \leq \rho
    \end{aligned}
    \end{equation}
     Based on earlier assumptions and the eigenvalue inequality of the Hermite matrix, we can further deduce to
    \begin{equation}
    \begin{aligned}
        \nabla^2 f_i(w)\succeq \mu+2\lambda(3\mu^2-G\rho)\cdot w^\top w-8\lambda \mu \theta_i.
    \end{aligned}
    \end{equation}
    When the value of $\lambda$ conforms to inequality (16), then we get,
    \begin{equation}
    \begin{aligned}
    \nabla^2 f_i(w)\succeq 6\lambda\mu^2\cdot w^\top w+\mu'\succeq \mu'.
    \end{aligned}
    \end{equation}
    According to the second-order properties of convex functions, it can be demonstrated that inequality (\ref{convex-mu}) holds.
    \\
    Meanwhile since the first derivative of $f_i(w)$ is:
    \begin{equation}
    \begin{aligned}
        \nabla f_i(w)=& \nabla R_i(w)\\ &+2\lambda \nabla R_i(w)^\top w[\nabla^2R_i(w) w+\nabla R_i(w)],
    \end{aligned}
    \end{equation}
    when $w=w^*_i$, thereby $\nabla R_k(w^*_i)=0$, we have
    \begin{equation}
    \begin{aligned}
        \nabla f_i(w^*_i)=& \nabla R_k(w^*_i)
        \\ &+2\lambda \nabla R_i(w^*_i)^\top w^*_i[\nabla^2R_i(w^*_i)\cdot w^*_i]
        \\ &+2\lambda \nabla R_i(w^*_i)^\top w^*_i[\nabla R_i(w^*_i)]\\=&0.
        \label{optimal}
    \end{aligned}
    \end{equation}
    Equality (\ref{optimal}) is equivalent to that the unique optimal solution of the function $f_i(w)$ is $w^*_i$ which is the optimal solution of $R_i(w)$
    
\end{proof}
This implies that when designing the loss function, careful consideration must be given to the value of $\lambda$. Indeed, due to the inherent difficulty in assessing the convexity and smoothness of the empirical risk loss, it is practically impossible to quantitatively compute an upper bound for hyperparameter $\lambda$. However, in numerical simulation, we observe that a relatively small $\lambda$ typically suffices to ensure the convergence of the algorithm, indicating that ensuring this upper bound is not particularly challenging.
Next, we will demonstrate that under this design of $\lambda$, the loss function incorporating the penalty term will simultaneously exhibit smoothness.
\begin{theorem}
\label{L'}
(L'- Smoothness)\\
Based on (\ref{smoothness}), the below inequality holds, 
    \begin{equation}
        \begin{aligned}
            R(w_1)-R(w_2)\leq &<\nabla R(w_2),w_1 - w_2>\\
            &+ \frac{L'}{2}\|w_1-w_2\|^2,
        \end{aligned}
    \end{equation}
where,
    \begin{align}
       L'=L+10\lambda L^2 \beta^2 +4\lambda G^2+2G\rho \lambda \beta^2.
    \end{align}
    
\end{theorem}
When$\ \lambda$ is chosen appropriately, the loss function$\ f_k(w)$ is convex. Similarly, a smoothing hypothesis about the new loss function $R_k(w)$ can be given based on the existing hypothesis.

\begin{proof}
    Similarly to the proof of (\ref{convex-mu}), we investigate the second-order derivative of the loss function. According to (\ref{smoothness}), When $w_1=0$, we have
    \begin{align}
        \nabla R_k(w)^\top w\geq R_k(w)-R_k(0)+\frac{L}{2}w^\top w,
    \end{align}
    thus,
    \begin{equation}
    \begin{aligned}
        \nabla^2 f_k(w)\preceq & \nabla^2R_k(w)\\
        &+2\lambda[\nabla^2_kR(w)w][\nabla^2_kR(w)w]^\top \\
        &+2\lambda\nabla R_k(w)\nabla R_k(w)^\top\\
        &+8\lambda \nabla^2 R_k(w)[R_k(w)-R_k(0)+\frac{L}{2}w^\top w]\\
        &+2\lambda[\nabla R_k(w)^\top w][\nabla^3R_k(w) w]^\top.
    \end{aligned}
    \end{equation}
    Based on (\ref{smoothness}), (\ref{Bound}), (\ref{Third-order}) and the eigenvalue inequality of the Hermite matrix. The largest eigenvalue of the addition of two matrices is less than the sum of the largest eigenvalues of the two matrices, while $R_k(w)\leq R_k(0)$, we can further deduce to
    \begin{equation}
    \begin{aligned}
        \nabla^2 f_k(w)\preceq L+2\lambda(5L^2\cdot w^\top w+2G^2+ G\rho \cdot w^\top w).
    \end{aligned}
    \end{equation}
    According to the second-order properties of smooth functions, it can be demonstrated that inequality (\ref{L'}) holds. 
\end{proof}

Based on the above conclusions, we have established that the loss function incorporating the penalty term is convex and smooth. Additionally, in the centralized setting, grouping data by environment allows us to obtain the optimal invariant parameters under the current penalty term using gradient descent which is proved in Appendix-B. Next, we will shift our focus to the context of edge intelligence, analyzing the convergence properties of invariant learning within a federated framework.

In the server aggregation convergence analysis, we assume all selected $\mathcal{I}_{ID}$ devices are able to participate in every aggregation steps. Based on the exit strategy discussed earlier, we posit that there exists an upper bound on the heterogeneity differences among all participating clients in the training process. This is consistent with the assumptions commonly made in FL researches\cite{yeganeh2020inverse}.
\begin{Assumption} (Bounded Data Heterogeneity):
\\
For each client, there exists a constant $\phi \geq 0$, such that
\begin{equation}
    \|w^*_i - w^I\|^2 = \|w^e_i\|^2 \leq \phi^2
\end{equation}
\end{Assumption}

where $w^*_i$ represent the optimal parameter of local function $f_i(w)$
In the previous section, we have demonstrated that the loss function constructed in this paper guides the parameters towards a global invariant optimal parameter under a centralized framework. Based on these assumptions and conclusions, we will now prove that the loss function will still converge towards the optimal invariant parameter under a federated structure. This convergence is the foundation upon which our algorithm is built.

\begin{theorem} (The Upper Bound of Federated Loss Functions with Invariant Penalty)
\\
With the invariant penalty, the loss function $f_(w)$ holds
\begin{align}
    f(\hat{w}(T)) - f(w^I) \leq \frac{B^2}{2\eta T} + \frac{2L'}{\mu T^2}[ L' B^2 + (\mu'+L) \phi^2 ] 
\end{align}
where $\hat{w}(T) = \frac{1}{T}\sum_{t=1}^T w^{t+1}$, representing the running average of output global parameters in the center.
\end{theorem}
\begin{proof}
See detailed proof in Appendix C. Consider Theorem \ref{convex'} and Theorem \ref{L'}, we have
     \begin{align}
         f(w^t) - f(w^I) &\leq <\nabla f(w^t),w^{t+1} - w^*> + f(w^t) \notag \\ 
         &-f(w^{t+1}) + \frac{L'}{2} \|w^{t+1} - w^{t}\|^2  
    \end{align}
    By applying rearrangement to above inequality, we obtain 
    \begin{align}
         f(w^{t+1}) - f(w^I) \leq <\nabla f(w^t),w^{t+1} - w^*>   + \frac{L'}{2} \|w^{t+1} - w^{t}\|^2  
    \end{align}
    Since $(\|w^{t+1} - w^{t}\|+\|w^{t+1} - w^{*}\|)^2\geq 0$, $f(w^{t+1}) - f(w^I)$ has following quality 
    \begin{align}
        f(w^{t+1}) - f(w^I) & \leq \frac{1}{2\eta}(\|w^{t} - w^{I}\|^2-\|w^{t+1} - w^{I}\|^2)\notag\\  
        & + \frac{1}{2}(L'-\frac{1}{\eta}) \|w^{t+1} - w^{t}\|^2 
    \end{align}
    Due to the inequality of arithmetic and geometric means, we have
    \begin{align}
        \sum^T_{t=1} \|w^{t+1} - w^t\|^2 & \geq \frac{1}{T}(\| \sum^T_{t=1} \|w^{t+1}(w^{t+1} - w^{t}) )\notag \\  
        & = \frac{1}{T} \|w^{T+1} - w^{1}\|^2 
    \end{align}
    According to (\ref{convex-mu}). we have 
    \begin{align}
        f(\hat{w}(T))-f(w^I)\leq \frac{1}{T}\sum^T_{t=1}[f(w^{t+1}) - f(w^I)]
    \end{align}
    where $\hat{w}(T) := \frac{1}{T}\sum^T_{t=1}w^{t+1}$, which means the upper bound of loss function can be represent by 
    \begin{align}
        f(\hat{w}(T))-f(w^I)\leq\frac{B^2}{2\eta T}+\frac{L'}{2T^2}\|w^{T+1}-w^1\|^2\label{sum(f(w^t+1) - f(w^I))}
    \end{align}  
    Since $\nabla f_i(w^*_i)=0$ has been proved in (\ref{optimal}), we have
    \begin{align}
        \|w^t-w^*_i\|^2 \leq \frac{1}{\mu}[f_i(w^t)-f_i(w^*_i)] 
    \end{align}  
    So $\|w^{t}-w^I\|^2$ can be limited by following elements
    \begin{align}
        \|w^t-w^I\|^2& =  \sum_{i \in \mathcal{I}_{ID}}a_i \|w^t-w\|^2 \notag \\
        &\leq \sum_{i \in \mathcal{I}_{ID}}2a_i(\|w^t-w^*_i\|^2+\|w^I-w^*_i\|^2) \notag \\
        &\leq \sum_{i \in \mathcal{I}_{ID}}2a_i(\frac{2}{\mu'}f_i(w^t)-\frac{2}{\mu'}f_i(w^*_i)+\|w^I-w^*_i\|^2) 
    \end{align}
    Hence, $\|w^{T+1}-w^1\|^2$ can be bounded as
    \begin{align}
        \|w^{T+1}-w^1\|^2 \leq \leq \frac{4L'B^2}{\mu'}+(4+\frac{4L'}{\mu'})\phi^2
    \end{align}
    Plugging (\ref{update}) in (\ref{sum(f(w^t+1) - f(w^I))}), we have 
    \begin{align}
        \sum^T_{t=1}[f(w^{t+1}) - f(w^I)] &\leq\frac{B^2}{2\eta}+\frac{L'}{2T}(\frac{4L'B^2}{\mu'}+(4+\frac{4L'}{\mu'})\phi^2) \notag \\
        & = \frac{B^2}{2\eta}+\frac{2L'}{T\mu'}[L'B^2+(\mu'+L')\phi^2]
    \end{align}
    According to Theorem 2 (\ref{convex-mu}), we have 
    \begin{align}
        f(\hat{w}(T))-f(w^*) &\leq\frac{1}{T}\sum^T_{t=1}[f(w^{t+1}) - f(w^I)] \notag \\
        & = \frac{B^2}{2\eta T}+\frac{2L'}{T^2\mu'}[L'B^2+(\mu'+L')\phi^2]
    \end{align}    
\end{proof}
Therefore, after incorporating the invariance penalty term, the global optimum point shifts. In the federated framework, the model continues to converge at the same rate to the unique invariance optimum point.
\section{Simulation Results}
\newcolumntype{Y}{>{\centering\arraybackslash}X}

\begin{figure*}[htbp]
  \centering
  \begin{minipage}[b]{0.3\textwidth}
    \includegraphics[width=\textwidth]{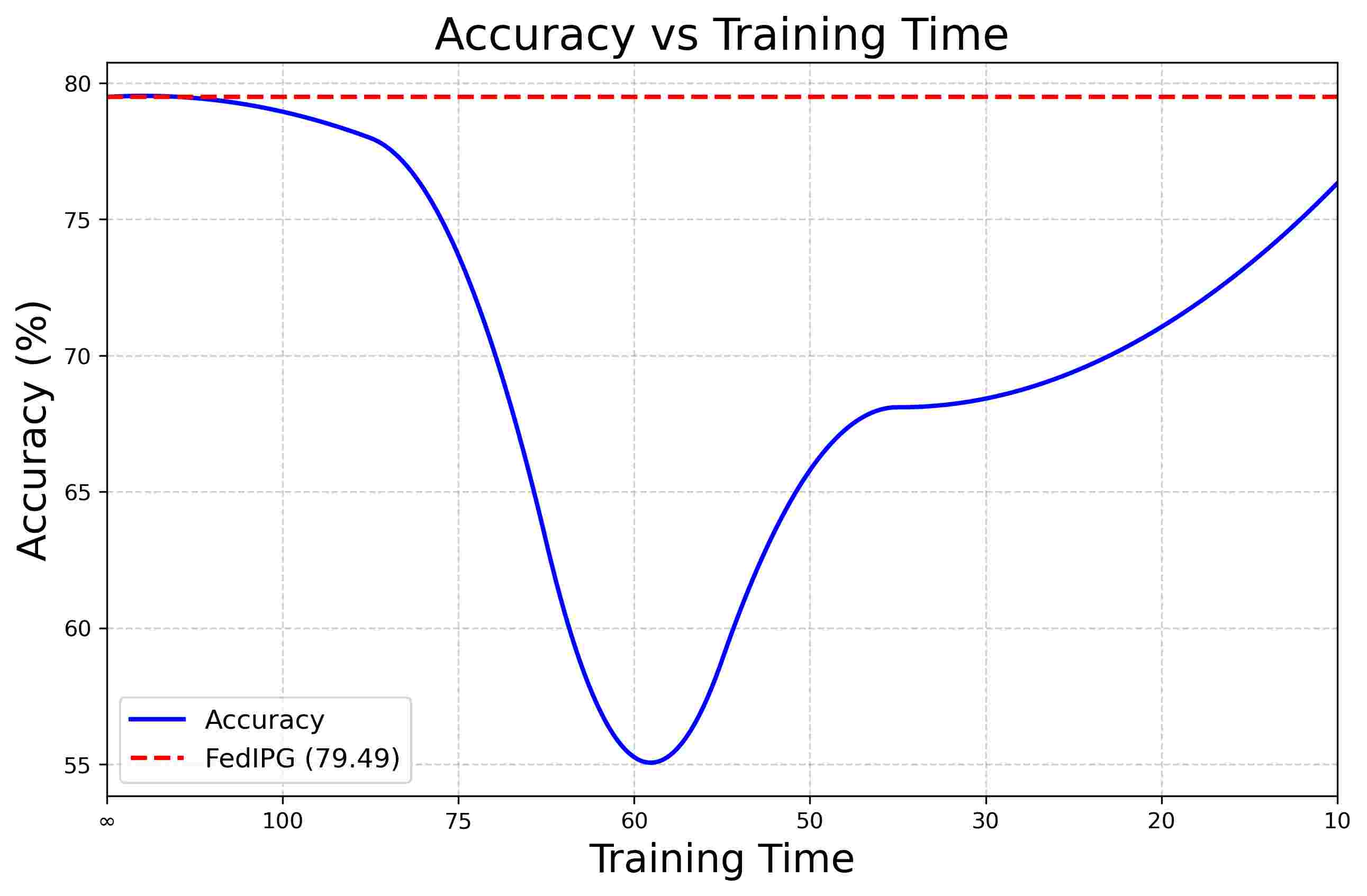}
  \end{minipage}
  \hfill
  \begin{minipage}[b]{0.3\textwidth}
    \includegraphics[width=\textwidth]{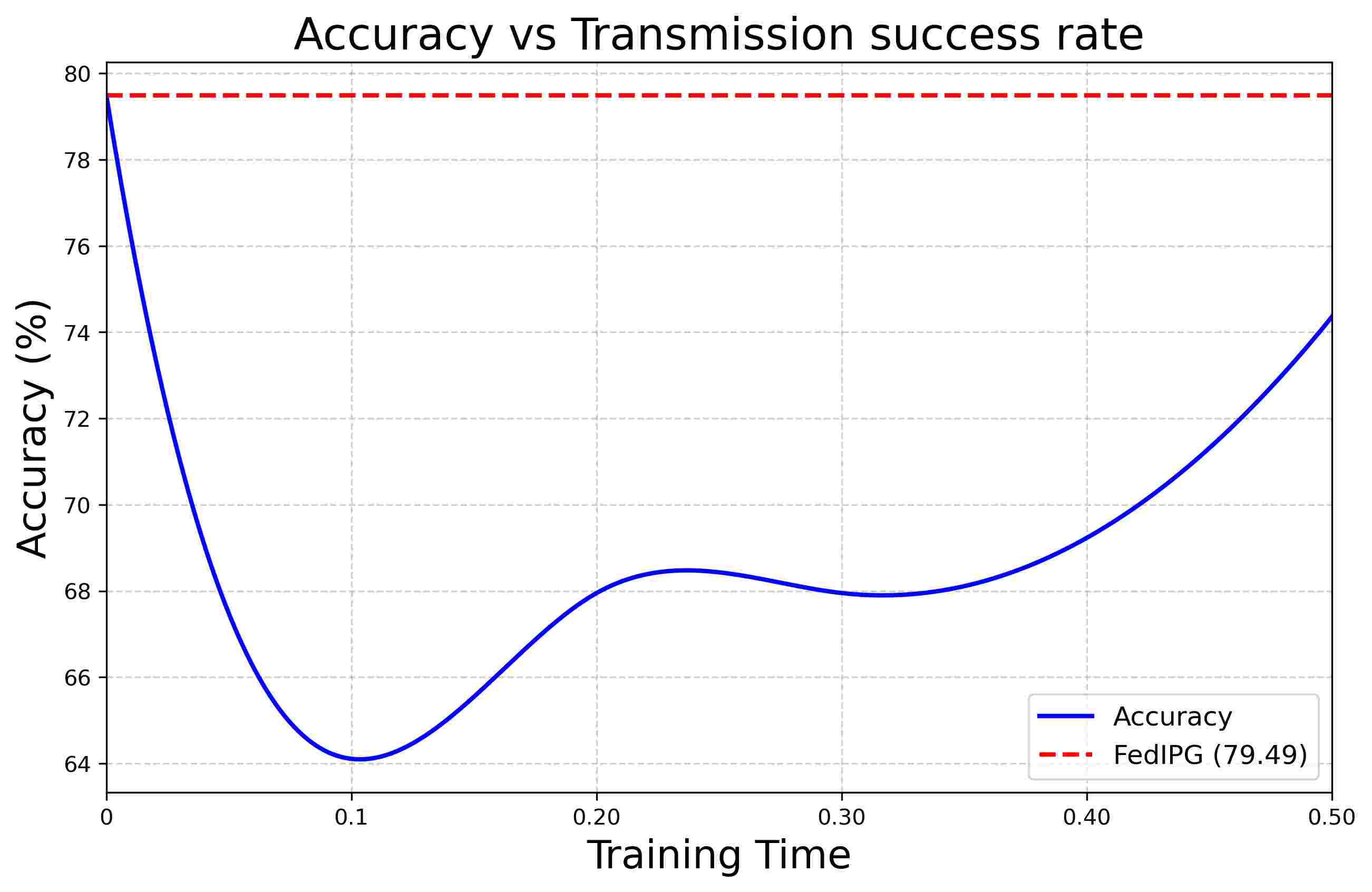}
  \end{minipage}
  \hfill
  \begin{minipage}[b]{0.3\textwidth}
    \includegraphics[width=\textwidth]{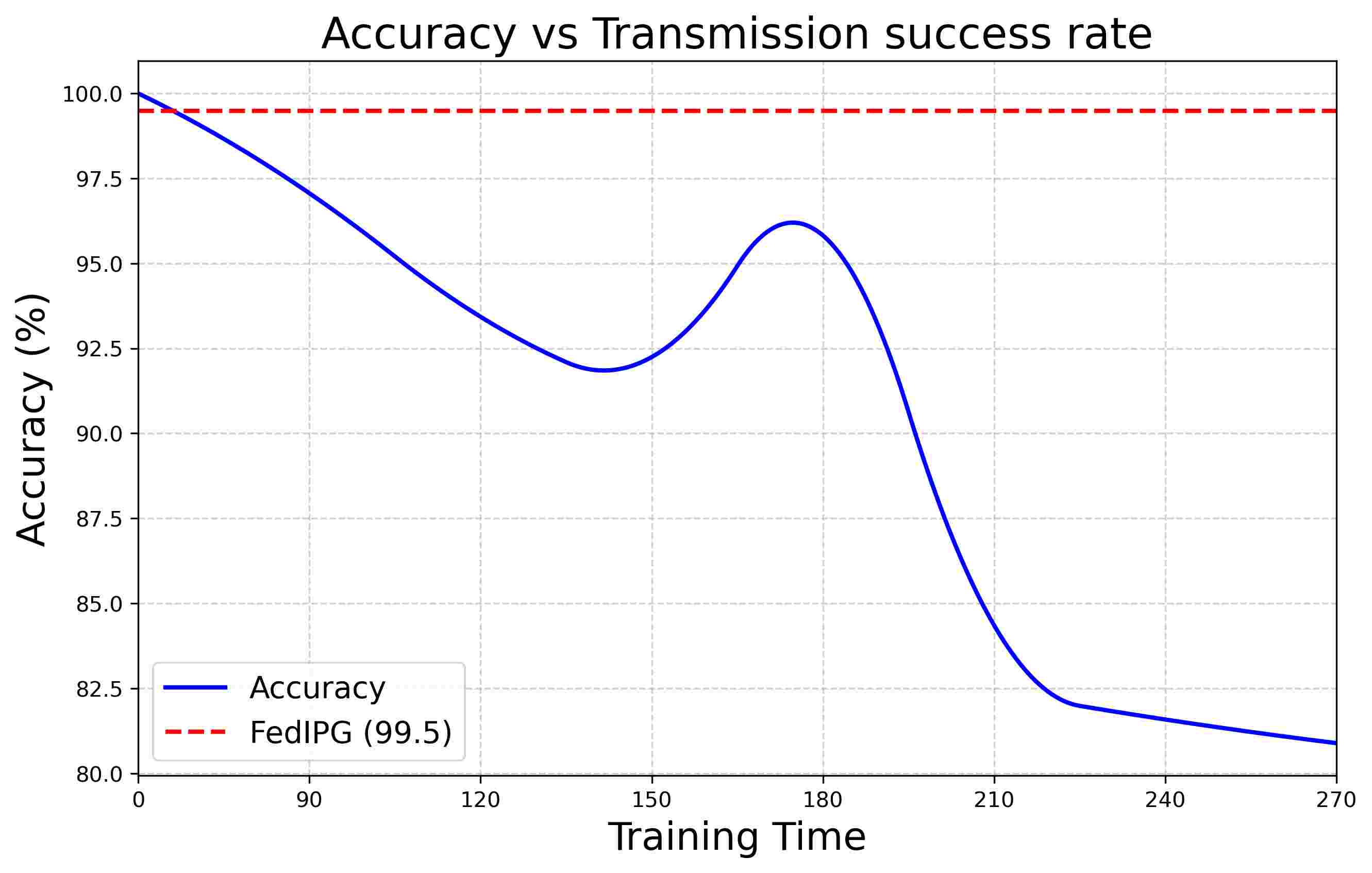}
  \end{minipage}

  \begin{minipage}[b]{0.3\textwidth}
    \includegraphics[width=\textwidth]{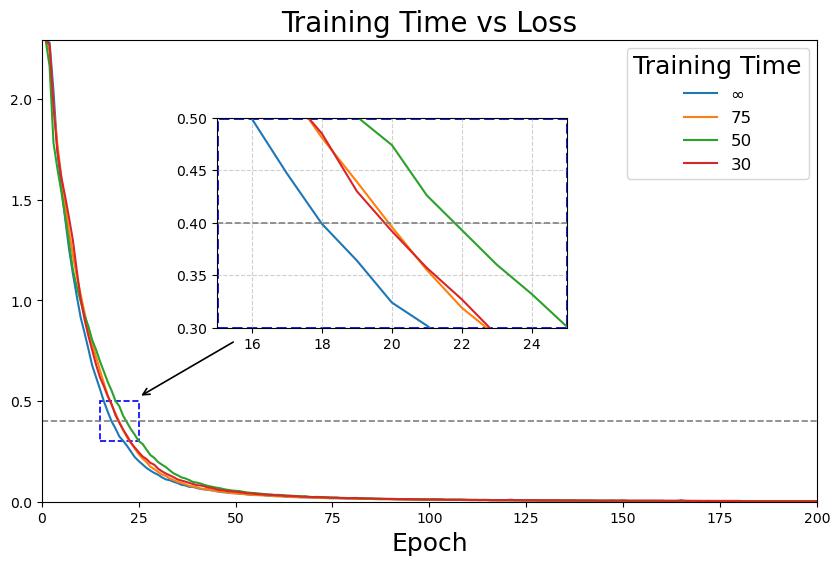}
  \end{minipage}
  \hfill
  \begin{minipage}[b]{0.3\textwidth}
    \includegraphics[width=\textwidth]
    {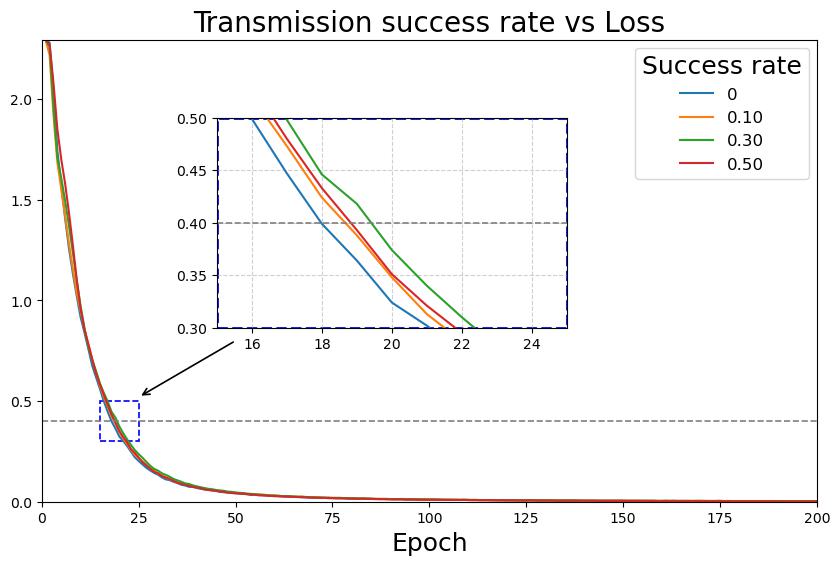}
  \end{minipage}
  \hfill
  \begin{minipage}[b]{0.3\textwidth}
    \includegraphics[width=\textwidth]{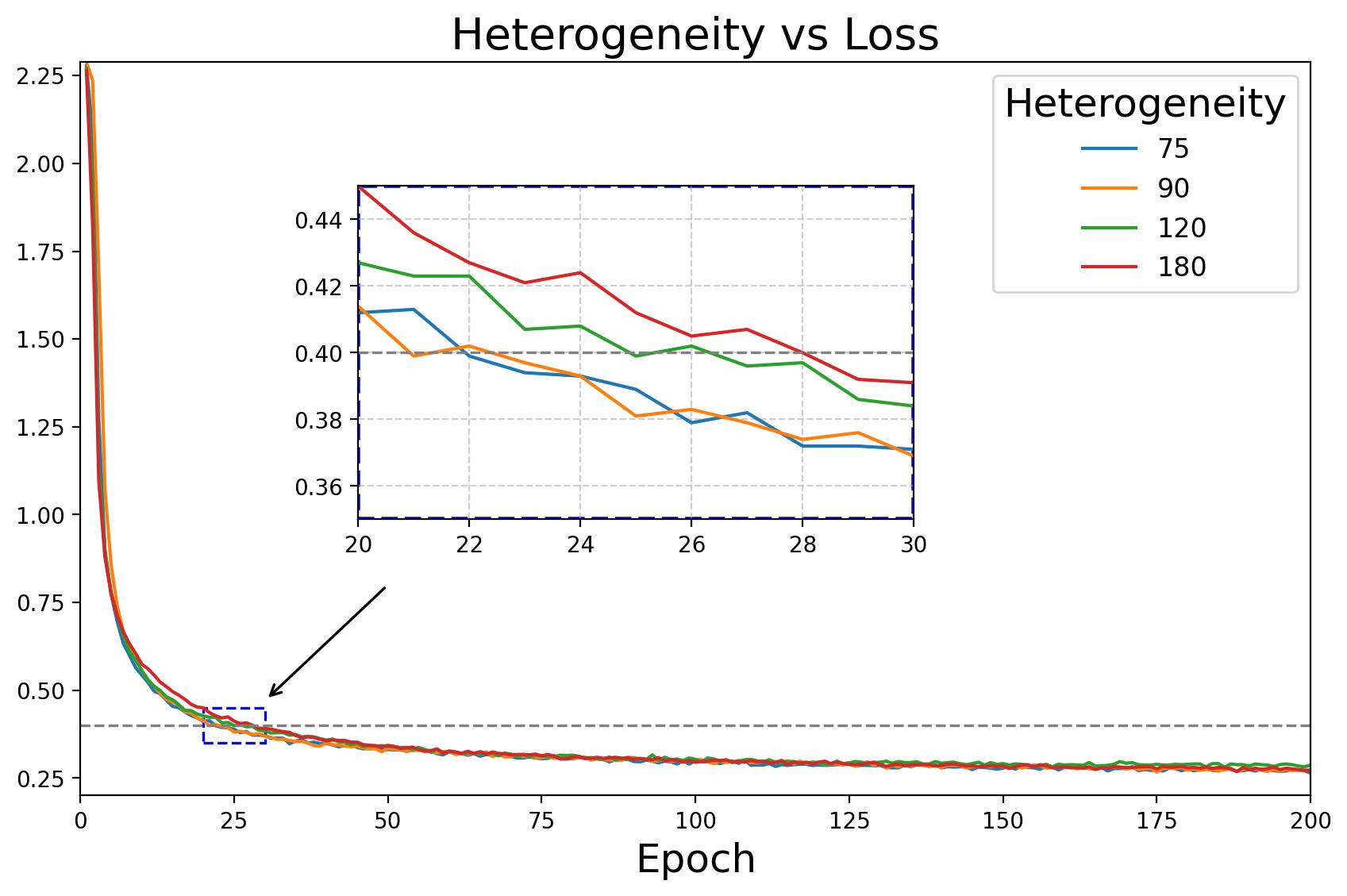}
  \end{minipage}
  \caption{Trend of accuracy and loss function changes as the degree of anomaly changes}
\label{exitresult}
\end{figure*}
In this chapter, we will discuss the performance of the invariant federated learning algorithm proposed in this paper in edge intelligence scenarios in three parts. Firstly, we will discuss the impact of heterogeneous and asynchronous clients, and demonstrate that the proposed algorithm can provide basic performance guarantees for non-participating clients. Secondly, we will compare the generalization ability of our algorithm with the current state-of-the-art algorithms. Finally, we will test the generalization effect of the proposed algorithm in mixed environments, where models with anti-confounding capabilities can be considered to have preliminary causality.

\subsection{Exit Strategy for Aberrant Clients }

\begin{table*}[htbp]
\caption{Average test accuracy using leave-one-out domain validation in the scenario with different number of clients.}
\centering
\renewcommand{\arraystretch}{1.5}
\begin{tabularx}{\textwidth}{|c|Y|Y|Y|Y|Y|Y|}
\hline
\diagbox{Algorithms / $\%$}{Datasets-clients} & RotateMNIST- 5& RotateMNIST-50  & PACS-5 & PACS-50 & VLCS-5 & VLCS-50 \\
\hline
FedAvg & $94.32\pm 0.09$ & $90.82\pm 0.57$  & 81.33 $\pm$ 0.01& 65.91 $\pm$ 0.02 &76.15 $\pm $ 0.01 & 70.12  $\pm$ 0.01\\
\hline
FedIIR & $95.33\pm 0.14$ & $94.20\pm 0.40$  & 81.23 $\pm$ 0.01& 68.41 $\pm$ 0.02  &76.47 $\pm $ 0.01 & 74.50  $\pm$ 0.01\\
\hline
FedIPG (ours) & \textbf{96.40} $\pm$ 0.20 & \textbf{95.68} $\pm$ 0.43 &\textbf{81.75} $\pm$ 0.01 &\textbf{69.34} $\pm$ 0.01  &\textbf{76.53} $\pm $ 0.01 & \textbf{75.27} $\pm$ 0.01\\
\hline
\end{tabularx}

\label{looresult}
\end{table*}

In this experiment, we aim to empirically prove that the exit strategy proposed in the theoretical analysis is beneficial to the global system. The preceding theoretical analysis has demonstrated that the participation of clients with delays and those exhibiting strong heterogeneity can impede the aggregation process. The global setup comprises 20 clients, among which one is designated as an aberrant client. Three distinct anomalous scenarios are considered, with the accuracy on the training set serving as a metric for aggregation efficacy, and the reduction of loss to a relatively low value as an indicator of aggregation speed. 

In the first scenario, the anomalous client requires more than one round to compute new local parameters. Utilizing a Non-iid CIFAR-10 dataset and a CNN model, we investigate the impact of the anomalous client's training duration on aggregation. The second scenario involves an anomalous client that successfully receives and transmits parameters with a certain probability. Employing the same dataset and model configuration as the first scenario, we examine the effect of transmission success rate on aggregation. Considering that the impact of the first and second types of anomalous clients on the loss function becomes less observable in the later stages of training, it is assumed that all clients participate in at least one round of aggregation at the beginning. The third scenario features an anomalous client with pronounced heterogeneity, constructed by rotating images from a handwritten digit dataset, with a CNN model employed. Stable participating clients are subjected to rotations of 0, 15, 30, 45, 60, and 75 degrees, while the anomalous client's degree of heterogeneity is considered stronger the greater the deviation from these angles. We assess the influence of heterogeneity intensity on aggregation performance.

In all three experiments, training without the anomalous client serves as the benchmark, against which the impact of the anomalous conditions is compared. From the analysis in figure \ref{exitresult}, it can be observed that the participation of anomalous clients in training affects the aggregation speed and final performance of the model to varying degrees. Compared to scenarios where such clients do not participate at all, ``slow'' clients degrade the model's performance on the training set. Clients with extremely poor training capabilities, due to their infrequent participation, have a relatively minor detrimental effect on performance. As the training capability of clients improves, which means the number of rounds required for training decreases, the performance initially worsens significantly before slowly recovering, aligning with the theoretical analysis. At its worst, the performance degradation exceeds 35$\%$. Additionally, slow clients significantly reduce the aggregation speed. Similar to the first anomalous scenario, the degradation caused by limited communication environments is pronounced, showing an initial decline followed by a rise, although under our current parameter settings, a recovery to benchmark accuracy levels has not been observed. In the third anomalous scenario, as the degree of anomaly increases, the final model accuracy significantly decreases, with a maximum degradation of up to 20$\%$. Unlike the first two scenarios, a slight recovery in accuracy is noted when digits are rotated by 180 degrees, which we tentatively attribute to the fact that multiple 180-degree rotations align with the original image orientation. In terms of convergence speed, when the heterogeneity of anomalous clients is minimal, there is no significant hindrance to aggregation. However, as the degree of heterogeneity increases, its impact on aggregation speed becomes markedly evident. Furthermore, the red dashed line in the accuracy graph represents the performance of our proposed FedIPG algorithm on participating training clients. When anomalous clients do not participate in training, FedIPG does not exhibit significant performance degradation compared to traditional algorithms.

\begin{figure}[!t]
\centering
\includegraphics[width=3in]{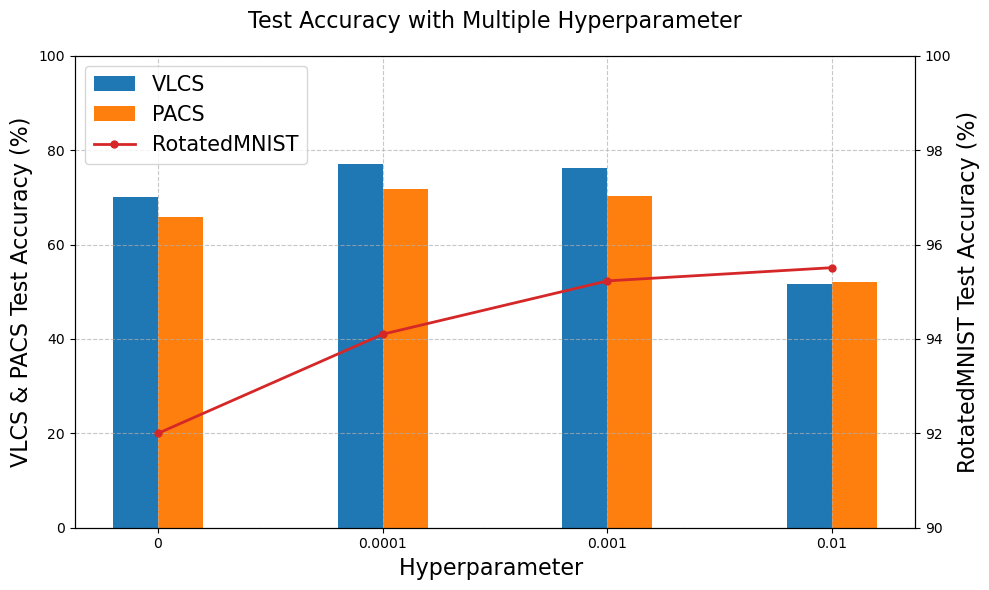}
\captionsetup{font={small}}
\caption{Average test accuracy with different hyperparameter $\lambda$.}
\label{hyperparameter}
\end{figure} 

\begin{figure}[!t]
\centering
\includegraphics[width=3.2in]{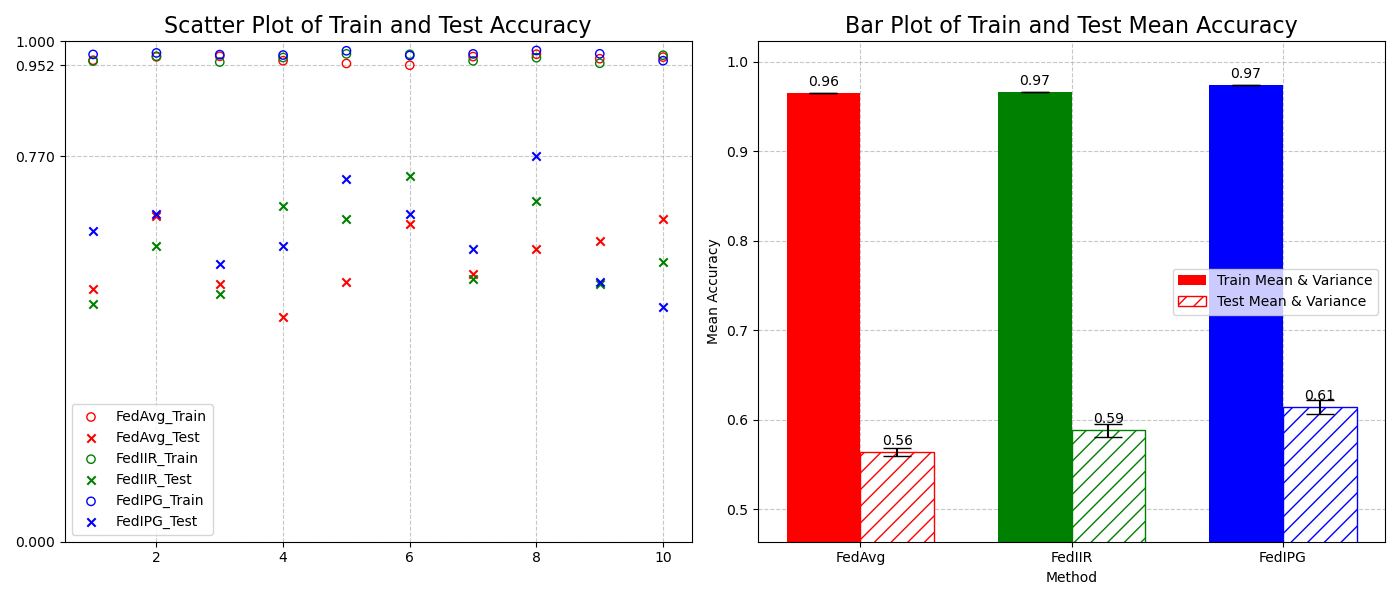}
\captionsetup{font={small}}
\caption{Train and test accuracy using causal validation with different quality of confounding in CNN network.}
\label{cmnist_cnn}
\end{figure} 
In summary, we have demonstrated that anomalous clients adversely affect both the final performance and the learning speed of model training. Therefore, provided that sufficient data volume can be ensured, an exit strategy is viable to safeguard the performance and learning speed for the majority of clients. Next, we will ensure through the FedIPG algorithm that exiting clients can achieve satisfactory performance, even in the presence of heterogeneity. 

\subsection{Leave-One-Domain-Out Generalization}
\begin{figure*}[!t]
\centering
\includegraphics[width=7in]{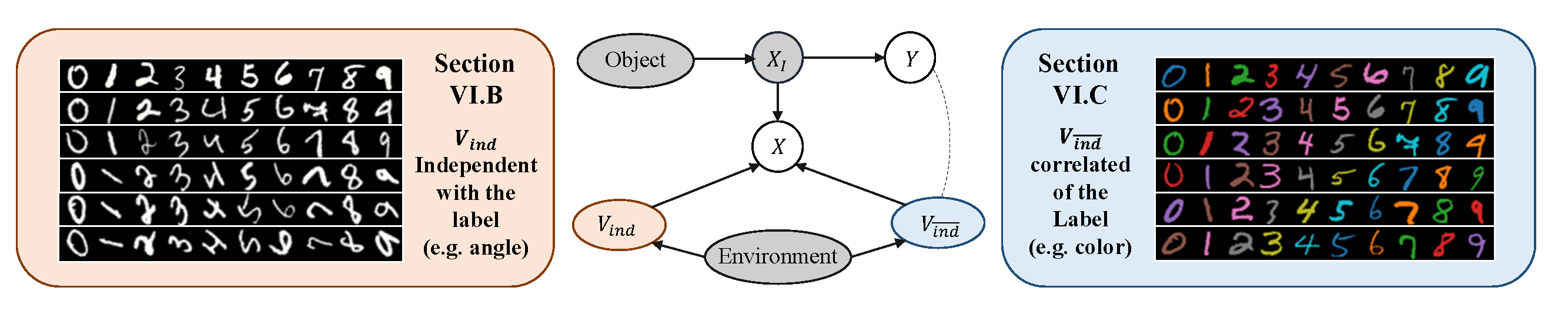}
\captionsetup{font={small}}
\caption{Two different experimental designs verify the generalization ability of the model when the environment variables are label dependent and independent.}
\label{experiment}
\end{figure*}

\begin{figure*}[htbp]
  \centering
 
  \begin{minipage}[b]{0.3\textwidth}
    \includegraphics[width=\textwidth]{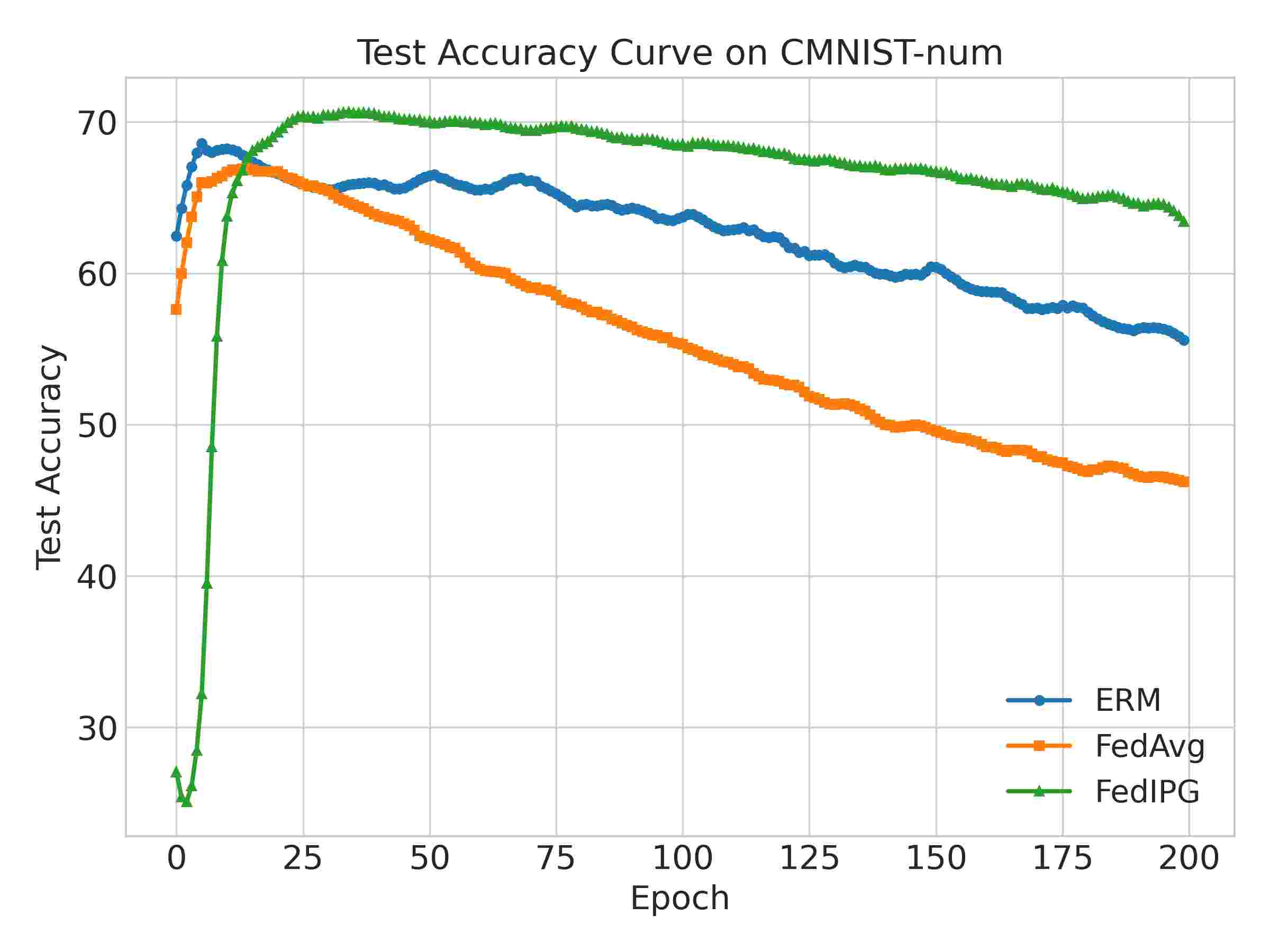}
  \end{minipage}
  \hfill
  \begin{minipage}[b]{0.3\textwidth}
    \includegraphics[width=\textwidth]{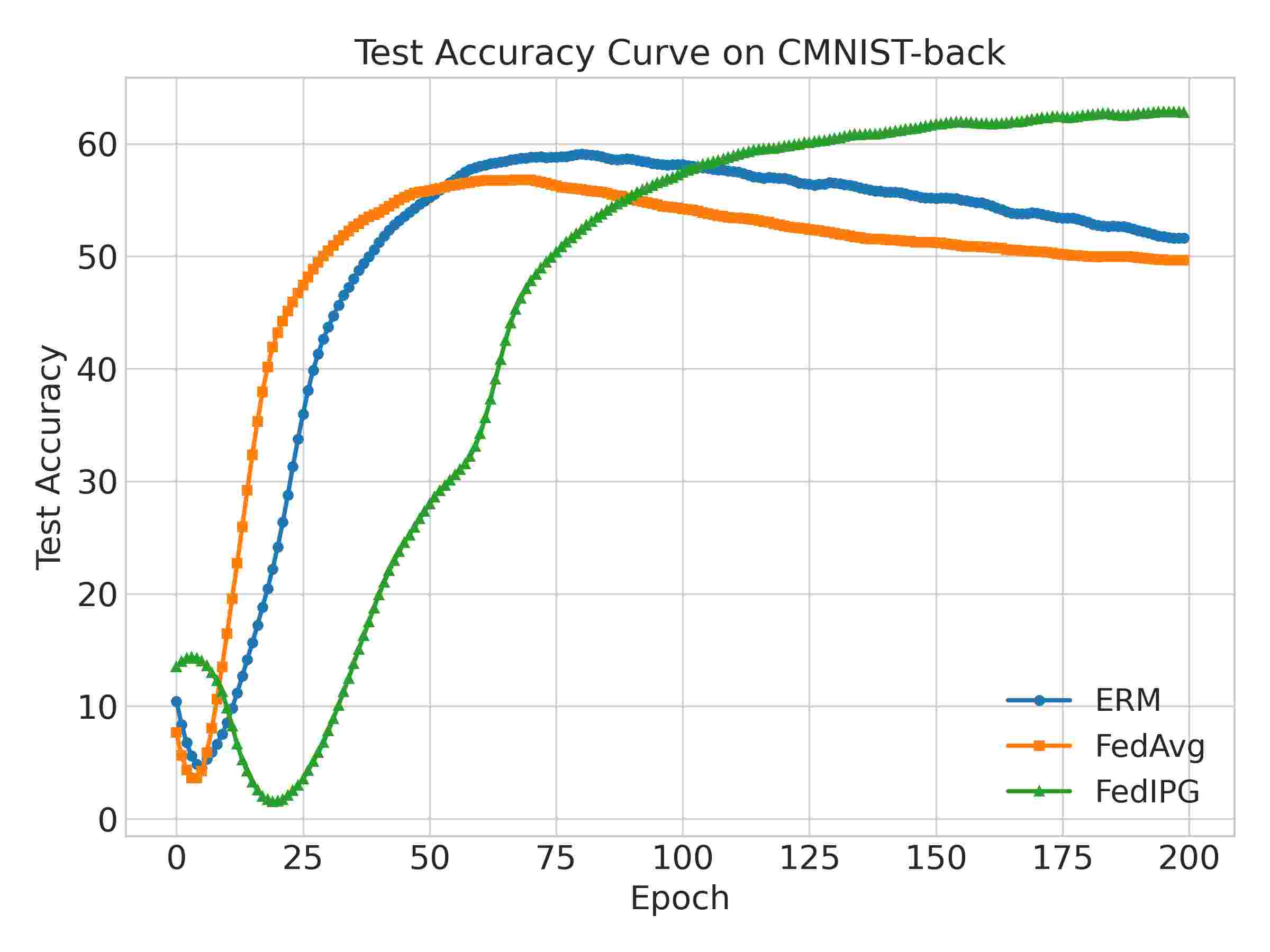}
  \end{minipage}
  \hfill
  \begin{minipage}[b]{0.3\textwidth}
    \includegraphics[width=\textwidth]{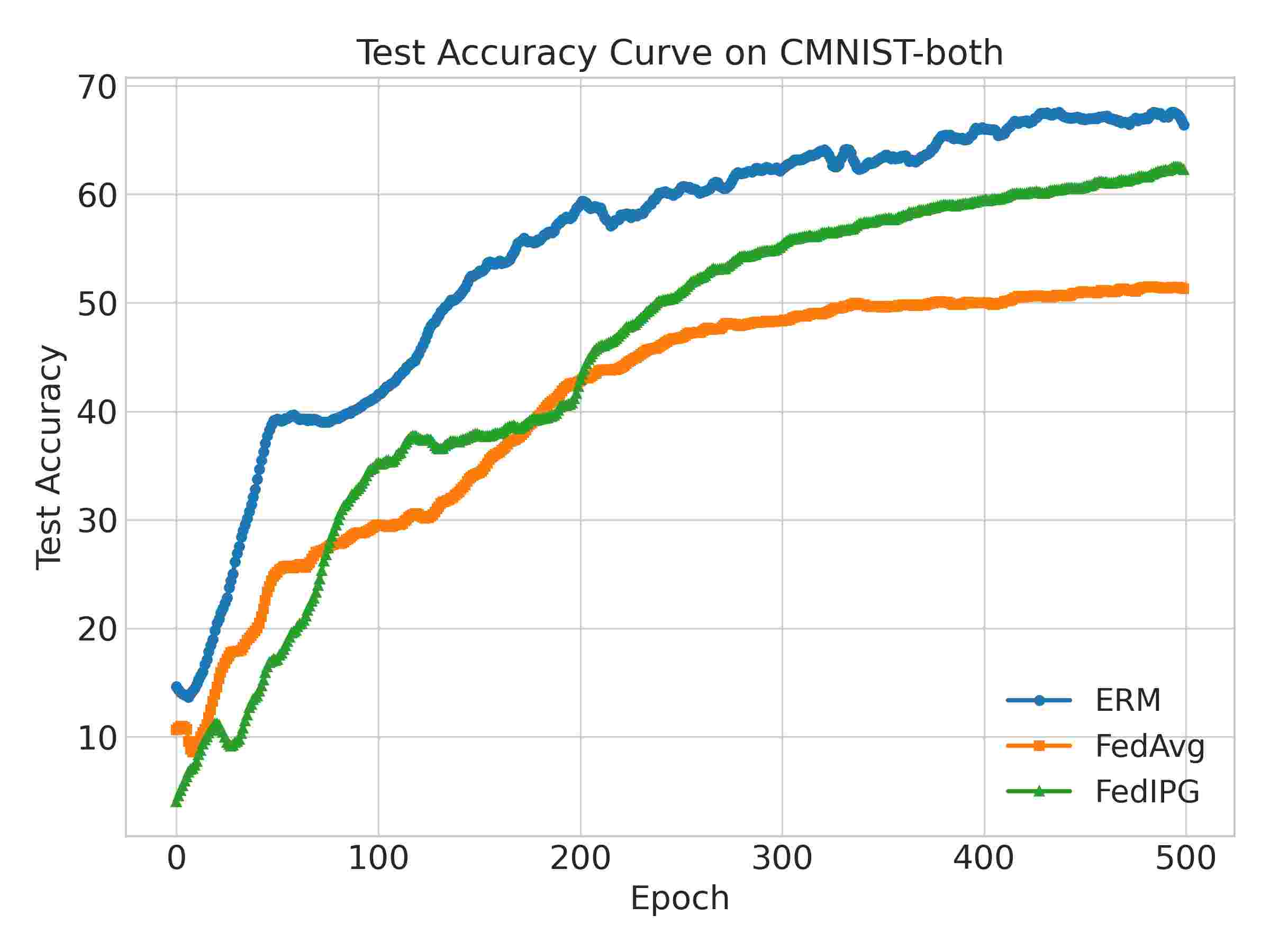}
  \end{minipage}

  \caption{The accuracy curves of three algorithms based on inference in data with different quality of confounding.}
\label{cmnist_mlp_curve}

\end{figure*}

In our experiments on domain generalization and OOD generalization, we adopt a “leave one domain out” design. The dataset is divided into several domains based on categories, with one domain withheld during training and used solely for OOD testing (see Fig. \ref{experiment}). For instance, the PACS dataset comprises four style-based domains (art painting, cartoon, photo, sketch), each containing seven classes (dog, elephant, guitar, giraffe, horse, house, person). The model is trained on three domains to build a seven-class classifier, which is then evaluated on the held-out domain.

We validated our approach using three datasets: RotateMNIST, PLCS, and VLCS. Different network architectures were used per dataset: AlexNet for RotateMNIST and ResNet for VLCS and PACS. Two comparative algorithms were employed: FedAvg and FedIIR (the state-of-the-art federated algorithm for OOD generalization). Notably, FedIPG requires only half the communication rounds of FedIIR, making it more robust in environments with unstable communication.

To simulate varying scales of edge intelligence, we designed experiments with 5 and 50 clients, each holding data from a single domain. As the number of clients increases, the data per client decreases, which necessitates stronger aggregation. We allocate 90$\%$ of each client's data for training and the remaining 10

Table \ref{looresult} shows that FedIPG outperforms all baselines across datasets and client scales. For simpler tasks with minimal OOD discrepancies, invariance processing yields modest improvements over conventional federated learning methods, albeit with increased variance. In contrast, for more complex tasks, invariance-based methods achieve nearly a 5$\%$ performance gain over traditional algorithms, and they outperform current state-of-the-art approaches, particularly in scenarios with many clients. Thus, FedIPG enhances OOD generalization without incurring extra communication overhead.

In our theoretical analysis, the hyperparameter governing the trade-off between the penalty term and the loss function is constrained within a limited range. We experimentally assessed its impact on training across three datasets with 50 clients. As shown in Fig. \ref{hyperparameter}, when the hyperparameter is within (0.0001, 0.01), FedIPG significantly improves performance. However, for complex datasets and networks, larger hyperparameter values degrade performance noticeably, as illustrated by the bar chart. Simpler tasks and models exhibit performance gains over a broader hyperparameter range, as seen in the line chart. Consequently, for more complex tasks, it is advisable to select a smaller hyperparameter value to avoid degradation.

\subsection{Anti-Confounding Causal Generalization }
\begin{table}[!t]
  \centering
  \caption{Average test accuracy using causal validation with different quality of confounding in MLP network.}

  \renewcommand{\arraystretch}{1.5} 
  \begin{tabular}{|c|c|c|c|}
    \hline
    \diagbox{Algorithms}{Stain} & Foreground & Background & Both \\
    \hline
    ERM & $55.78$ \tiny{$\pm 2.61$} & $52.79$ \tiny{$\pm 1.31$} & \textbf{66.46} \tiny{$\pm 4.07$} \\
    \hline
    FedAvg & $45.73$ \tiny{$\pm 1.51$} & $49.71$ \tiny{$\pm 1.23$} & $52.78$ \tiny{$\pm 4.16$} \\
    \hline
    FedIPG(ours) & \textbf{63.99} \tiny{$\pm$ 1.48} & \textbf{62.77} \tiny{$\pm$ 1.15} & 62.08 \tiny{$\pm$ 4.241} \\
    \hline
  \end{tabular}
\label{cmnist_mlp}
\end{table}
In this subsection, Algorithms are tested for learning nonlinear invariant predictors with a synthetic classification task derived from CMNIST. The goal is to predict a label assigned to each image based on the digit. Meanwhile, MNIST images are gray-scale, and the images' foreground, background, or both are colored with ten different colors that correlate strongly with the class label. The correlation between clients and the test set is spurious, as shown in Fig \ref{experiment}.

In the CNN network, a comparison is made with traditional federated learning algorithms and the FedIIR algorithm. The network structures of the three algorithms are identical to those used in the previous RotateMNIST experiment. We fix ten color and digit combinations, and in each experiment, we randomly allocate nine of these to 50 clients as the training set, while the remaining one serves as the test set post-training. As shown in Fig \ref{cmnist_cnn}, the average performance of FedIPG across ten tests is improved by $5\%$ compared to traditional algorithms, and it still shows a $2\%$ improvement over the current best out-of-distribution generalization model. Moreover, in the ten experiments where each of the ten combinations is used as the test set, FedIPG achieves the highest test accuracy five times, which is the highest among all three algorithms. This is a preliminary validation that, in a leave-one-domain-out test structure with added causally related confounding factors, invariance-based algorithms, especially FedIPG, maintain a significant advantage.

Further, a comparison is made in the MLP network with centralized algorithms and traditional federated learning algorithms. The foreground staining area is less than $50\%$, the background staining area is more than $50\%$, and images stained both in the foreground and background have a staining area of $100\%$. This increase in the staining area presents the interference strength of confounding factors. We set the target accuracy at $60\%$, requiring five clients and 600 samples per client for foreground staining, five clients and 2000 samples per client for background staining, and twenty clients with 2000 samples per client for foreground and background. As heterogeneity increases, the FedIPG algorithm, based on invariance penalty terms, can maintain anti-confounding capabilities by increasing the number of samples and clients. When the quality of the confounding features is not severe, conventional algorithms perform worse than the proposed, centralized or distributed algorithm, which is recorded in Table \ref{cmnist_mlp} and Fig. \ref{cmnist_mlp_curve}. This demonstrates that FedIPG has a specific capability for causal identification and anti-confounding, even though it does not explicitly model causal relationships and causal features. However, when it is significantly confounding, the performance of FedIPG, although still much better than FedAvg, is inferior to that of centralized algorithms. This indicates that FedIPG's causal inference capability remains relatively weak, and its performance does not match that of centralizing the data.

The results clearly illustrate that in scenarios characterized by significant heterogeneity among environments, ERM and FedAvg may not perform as expected due to the reversal of correlation direction in the test environment. FedIPG stands out by effectively extracting nonlinear invariant predictors across multiple environments within centralized and distributed frameworks. Most notably, the final accuracy of all four OOD tests underscores the reliable generalization capability of FedIPG. 

\section{Conclusion}
In this paper, we propose a novel approach to address the challenges of heterogeneity and asynchrony in federated learning. To ensure the aggregation effectiveness and speed for the majority of clients, it is advisable to allow aberrabt clients to exit the training process, thereby mitigating heterogeneity and system latency. This strategy is complemented by a model with out-of-distribution generalization capabilities, which is named FedIPG, ensuring baseline performance for the exiting clients. Our system achieves better overall performance for both training clients and out-of-distribution clients. Additionally, we have discovered that, with the support of invariance, FedIPG exhibits preliminary causal properties. Exploring how to leverage more advanced causal relationships within federated systems to enhance generalization represents a potential direction for future research.
\bibliographystyle{IEEEtran}
\bibliography{IEEEabrv,Ref.bib}

\end{document}